\documentclass[nonacm,sigconf]{acmart}
\usepackage{popets}

\usepackage{array}
\usepackage{pifont}
\usepackage{amsmath}
\usepackage[ruled,linesnumbered]{algorithm2e}
\usepackage{subcaption}
\usepackage{graphicx}
\usepackage[dvipsnames]{xcolor}
\usepackage{multirow}
\usepackage{enumitem}
\usepackage{wasysym}
\setcopyright{popets}
\copyrightyear{}

\acmYear{}
\acmVolume{}
\acmNumber{}
\acmDOI{}
\acmISBN{}
\acmConference{}
\begin{document}

\title[Pushing the (Decision) Boundaries]{Pushing the (Decision) Boundaries: Dynamically Calibrating Differentially Private Noise to Explainability in Federated Learning}

\author{Michael Khavkin}
\orcid{}
\affiliation{%
  \institution{School of Industrial \& Intelligent Systems Engineering, Tel Aviv University}
  \city{}
  \state{}
  \country{}}
\email{}

\author{Kichang Lee}
\affiliation{%
  \institution{School of Integrated Technology, Yonsei University}
  \city{}
  \country{}}
\email{}

\author{Jaeho Jin}
\affiliation{%
  \institution{School of Integrated Technology, Yonsei University}
  \city{}
  \country{}
}
\email{}

\author{JeongGil Ko}
\affiliation{%
 \institution{School of Integrated Technology, Yonsei University}
 \city{}
 \state{}
 \country{}}
\email{}

\author{Eran Toch}
\affiliation{%
 \institution{School of Industrial \& Intelligent Systems Engineering, Tel Aviv University}
 \city{}
 \state{}
 \country{}}
\email{}

\renewcommand{\shortauthors}{}

\begin{abstract}

Federated Learning (FL) with Differential Privacy (DP) is increasingly adopted to preserve data confidentiality in distributed machine learning. However, DP noise distorts learned representations and degrades explanation fidelity, limiting differentially private FL where trustworthy explanations are required, such as assistive clinical diagnosis. Prior work adapted DP noise with static feature-importance signals, restricting explainability to post hoc analysis and precluding noise calibration to explanation quality during training. We propose XCal-FL, a closed-loop, explainability-driven local training algorithm for image classification in cross-silo FL that dynamically calibrates DP noise from three complementary signals: (1) prediction logit variations, measuring causal influence on model confidence, (2) counterfactual margins, capturing decision-boundary sensitivity, and (3) saliency concentration, quantifying spatial coherence of model attention, while enforcing formal DP guarantees via adaptive privacy accounting. Experiments on three medical imaging datasets across varying FL configurations show that XCal-FL yields more accurate and interpretable global models, improving predictive performance by over 10\% and explanation fidelity by up to 5$\times$ over static-noise FL, and outperforming state-of-the-art adaptive DP methods in fidelity. XCal-FL also achieves higher privacy-budget efficiency, turning each unit of cumulative privacy loss into larger gains in both accuracy and explanation fidelity. Our analysis further reveals that, unlike predictive performance, which scales roughly linearly with privacy loss, explanation fidelity exhibits non-linear dynamics. These findings suggest explainability is a distinct dimension of the privacy trade-off that cannot be inferred from utility alone, with implications for training and privacy-budget allocation in decision-critical applications.

\end{abstract}

\keywords{Federated Learning, Explainable AI, Differential Privacy, Privacy-Utility Tradeoff, Explanation Fidelity, Medical Imaging}

\maketitle

\section{Introduction}
\label{sec:intro}

Machine learning (ML) models are increasingly deployed in domains that require learning from distributed and sensitive data, motivating the adoption of privacy-preserving training paradigms, most notably Federated Learning (FL) \cite{mcmahan2017communication}.
FL enables collaborative model training across multiple parties without centralizing raw data, helping organizations comply with privacy regulations such as GDPR \citep{GDPR2016} and HIPAA \cite{HIPAA1996}.
Particularly, in healthcare applications, ML models for processing medical images provide rich diagnostic information that supports early disease detection and personalized treatment planning, while reducing the burden on clinicians through automated analysis \cite{anikwe2022mobile,li2023privacy}.

Despite its privacy advantages, FL alone does not fully prevent information leakage, as shared model updates can still be exploited through inference and reconstruction attacks \cite{zhu2019deep,geiping2020inverting}.
To address this risk, Differential Privacy (DP) \citep{dwork2006calibrating} has emerged as the gold standard for privacy protection in FL \citep{xu2024privatetraininglanguage}, typically implemented via gradient perturbation mechanisms such as DP-SGD \cite{abadi2016deeplearningwithdp}.
DP provides formal guarantees against information leakage by bounding the sensitive information that can be inferred about the input data or its subjects from model outputs and measured by the privacy loss budget $\varepsilon$.
However, injecting DP noise during FL training introduces trade-offs that extend beyond the privacy-utility trade-off \citep{alvim2012differential}.
In particular, stochastic noise can distort learned representations, weakening the relationship between input features and model predictions and degrading the fidelity of model explanations, which are critical for understanding \textit{why} a model produced a given output.
Moreover, while predictive performance degrades approximately monotonically with increasing noise, the relationship between DP and explanation fidelity is less understood, suggesting that explainability may require dedicated treatment beyond standard privacy-utility analysis.
This is because gradient-based attributions depend on the fine-grained spatial coherence of the learned gradient field, a quantity that is more sensitive to the directional corruption introduced by DP-SGD \cite{duan2025analyzing} than the coarse class-level output on which predictive accuracy relies \cite{woerl2023initialization}.
This creates a tension between privacy and explainability, limiting the applicability of differentially private FL in settings where transparent and trustworthy explanations are required.
For example, in multi-hospital collaborations for radiological decision-support systems, a model that achieves high sensitivity but provides unreliable explanations may fail institutional review requirements or undermine clinician confidence in automated findings, limiting real-world adoption regardless of predictive accuracy.

Current differentially private FL methods apply uniform noise to all gradient features.
This leads to two complementary problems: important features that carry discriminative signal are over-perturbed, degrading both predictive accuracy and explanation fidelity, while less informative features receive more noise than necessary, wasting privacy budget without meaningful protection gain.
Alternative approaches adjust noise using feature-importance signals~\citep{li2023balancing}, but rely on a fixed noise magnitude that does not adapt to training dynamics.
Moreover, recent empirical studies provide growing evidence of a privacy-explainability trade-off under DP \cite{saifullah2024privacy}.
In both cases, noise is applied in an open-loop fashion: it does not condition on whether the model's current attributions are actually faithful to its predictions.
As a result, noise calibration cannot respond to shifts in explanation quality during training, and the joint trade-offs among privacy, predictive performance, and explainability remain insufficiently explored in FL.
This gap is further underscored by recent surveys highlighting the lack of evaluation methodologies for explainability under DP~\citep{lopez2024interplay}.
This leads to an important question: \textbf{(RQ1) Can explainability-aware noise calibration during FL training improve both predictive performance and explanation fidelity while preserving formal DP guarantees?}

Beyond the established privacy-utility paradigm, recent work demonstrates that multiple model objectives interact non-trivially under DP \cite{hartmann2022privacy}.
For example, techniques designed to improve predictive performance under strict privacy budgets can inadvertently amplify model bias, revealing a complex tension between privacy, utility, and fairness \cite{hassanpour2025impact}.
Yet, while static evaluations confirm that DP noise inherently degrades explanation fidelity \cite{saifullah2024privacy,ezzeddine2024differential}, the \textit{dynamics} of this degradation and how explanation quality evolves as training progresses remain unexplored.
This leads to two additional open questions.
First, \textbf{(RQ2) can adaptive DP noise calibration translate cumulative privacy loss into larger gains in accuracy and explanation fidelity than static noise allocation?}
Second, if so, \textbf{(RQ3) does explainability follow the same trajectory as predictive performance under cumulative privacy loss, or does it constitute a fundamentally different dimension of the privacy-utility trade-off?}
Beyond improving accuracy and explanation fidelity as isolated metrics, understanding how explanation quality evolves during training has a direct operational consequence: if explainability peaks earlier than predictive performance under cumulative privacy loss, practitioners can terminate training at the point of maximum explanation gain rather than waiting for performance to plateau, thereby utilizing less cumulative privacy loss budget.
This makes the dynamics of explainability under DP a practically actionable quantity, rather than being merely a post-hoc diagnostic.

In this work, we address the aforementioned questions by integrating privacy-aware explainability into the local training process of clients in FL.
Following prior work that improves FL by adapting training dynamics \citep{jin2022accelerated,li2023revisiting}, we propose \textit{XCal-FL}, an explainability-driven FL framework for cross-silo FL environments that dynamically calibrates the magnitude of DP noise during local training using a composite signal.
Specifically, our approach dynamically regulates noise injection during training by monitoring three signals: (1) class logit variations induced by salient-region masking (capturing the causal influence of highlighted regions on the prediction); (2) counterfactual margins (capturing decision-boundary sensitivity to the removal of explanatory evidence); and (3) influential region concentration (capturing the spatial coherence of the model's attention).
By integrating explainability and decision-related feedback directly into training in a closed-loop manner, \textit{XCal-FL} reduces unnecessary perturbation of informative model representations while fully preserving DP guarantees, as the cumulative privacy cost of the adaptively scaled noise is tracked at every training step to enforce the target ($\varepsilon$, $\delta$)-DP budget.

We evaluated our approach on three medical imaging tasks, covering blood cell type classification, pneumonia detection from X-ray scans and melanoma detection from dermatologic scans.
Experiments span varying cross-silo FL configurations with different numbers of clients ($N = 3, 10, 50$) and data heterogeneity conditions, measuring both predictive performance and explanation fidelity using the Remove and Debias (ROAD) metric \cite{rong2022consistent}.
Across datasets and privacy budgets, \textit{XCal-FL} consistently improved both predictive performance and explanation fidelity over static DP-FL and the state-of-the-art adaptive method AGP \citep{li2023balancing}, while achieving higher privacy--utility and privacy--explainability efficiency, i.e., translating cumulative privacy loss into larger gains in accuracy and explanation fidelity.
Importantly, our empirical analysis reveals an asymmetry: while predictive performance scales approximately linearly with cumulative privacy expenditure, explanation fidelity exhibits non-linear dynamics.
This suggests that standard privacy-utility analyses are insufficient to capture the full impact of DP on model behavior, and that explainability requires dedicated evaluation and optimization as a distinct dimension of the privacy-utility trade-off.

Our main contributions are summarized as follows:
\begin{itemize}

\item \textbf{Closed-loop noise calibration:} We proposed \textit{XCal-FL}, a cross-silo FL training approach that dynamically calibrates DP noise magnitude based on explainability- and decision-related signals, while preserving formal $(\varepsilon, \delta)$-DP guarantees through adaptive privacy accounting.

\item \textbf{Empirical evaluation on medical imaging:}
Across three medical image classification datasets and varying FL configurations, \textit{XCal-FL} improves F1-score by over 10\% and ROAD explanation fidelity by up to $5\times$ over static DP-FL, while outperforming AGP \cite{li2023balancing} by over 10\% in F1-score and over 75\% in explanation fidelity, on average.

\item \textbf{Privacy-budget efficiency analysis:} We proposed a quantitative evaluation of privacy-budget efficiency, defined as the gain in F1-score or ROAD per unit of cumulative privacy loss.
\textit{XCal-FL} improves privacy-utility efficiency by up to 25\% and privacy-explainability efficiency by threefold over training with static DP noise.

\item \textbf{Privacy-explainability dynamics:} Our empirical analysis reveals that, while predictive performance scales approximately linearly with cumulative privacy loss, explanation fidelity exhibits non-linear dynamics, with ROAD scores potentially increasing early in training before the privacy budget is fully consumed.

\end{itemize}

Taken together, these contributions advance the practical viability of differentially private FL for safety-critical clinical decision-support applications, where regulatory and institutional requirements increasingly demand that automated diagnostic outputs be accompanied by faithful and interpretable justifications.

\section{Related Work}

Recent work demonstrates that DP interacts non-trivially with model properties beyond utility: DP-driven generalization can amplify bias~\cite{hassanpour2025impact}, post hoc privacy analysis can reclaim budget without sacrificing formal guarantees~\cite{hartmann2022privacy}, and alternative federated optimization formulations leverage synthetic samples alongside gradients to improve privacy-utility trade-offs~\cite{wang2024fedlapdp}.
However, DP noise systematically degrades prediction performance and explanation fidelity~\cite{ezzeddine2024differential,saifullah2024privacy}, with gradient-based attributions~\citep{selvaraju2017grad} proving particularly sensitive to perturbations~\citep{zaher2024manifold}.
Prior DP-FL applications in healthcare rely on post hoc explainability~\citep{briola2024federated,adnan2022federated,jiang2023client} or static noise schedules~\citep{daniel2025flemXAI}.
Moreover, recent surveys highlight the lack of consistent evaluation frameworks for explainability under DP-FL \citep{lopez2024interplay}.
These findings motivate training-time approaches that explicitly account for explainability, rather than relying on post hoc analysis.

\newcommand{\symA}{\scalebox{0.65}{$\CIRCLE$}}
\newcommand{\symP}{\scalebox{0.65}{$\LEFTcircle$}}
\newcommand{\symF}{\scalebox{0.65}{$\Circle$}}
\newcommand{\qc}[1]{\makebox[1.5em][c]{#1}}
\newcommand{\adapted}{\qc{\symA}}
\newcommand{\placed}{\qc{\symP}}
\newcommand{\fixedq}{\qc{\symF}}

\begin{table*}[t]
\centering
\small
\caption{Comparison between adaptive DP-FL methods. \textbf{Adapted quantity}: clipping bound $C$, noise magnitude $\sigma$, budget allocation $\varepsilon$ (\symA{} adapted, \symP{} placement only, \symF{} fixed). \textbf{Feedback}: \emph{closed} = conditioned on a measured model response; \emph{open} = fixed schedule or one-shot read of model state. $^\dagger$Signal privatized and charged to the budget.}
\label{tab:relwork}
\resizebox{\textwidth}{!}{%
\begin{tabular}{@{} l ccc l l cccc @{}}
\toprule
& \multicolumn{3}{c}{\textbf{Adapted Quantity}} & & & \multicolumn{3}{c}{\textbf{Adapted At}} & \\
\cmidrule(lr){2-4} \cmidrule(lr){7-9}
\textbf{Method} & $C$ & $\sigma$ & $\varepsilon$ & \textbf{Adaptation Signal} & \textbf{Feedback} & \textbf{Side} & \textbf{Freq.} & \textbf{DP Granularity} & \textbf{Expl.-Aware} \\
\midrule
Adaptive clipping~\cite{andrew2021differentially} & \adapted & \fixedq & \fixedq & Private quantile of update norms$^\dagger$ & Closed & Server & Round & User-level & No \\
Adap DP-FL~\cite{fu2022adap} & \adapted & \adapted & \fixedq & Grad-norm heterogeneity + round index & Open & Client & Round & Record-level & No \\
Budget alloc.~\cite{hartmann2022} & \fixedq & \fixedq & \adapted & Mechanism output & Post hoc & --- & --- & --- & No \\
AGP~\cite{li2023balancing} & \fixedq & \placed & \fixedq & Grad-CAM channel importance & Open & Client & Round & Record-level & Yes \\
\textbf{XCal-FL (ours)} & \fixedq & \adapted & \fixedq & Composite Grad-CAM signal$^\dagger$ & Closed & Client & Batch & Record-level & \textbf{Yes} \\
\bottomrule
\end{tabular}%
}
\end{table*}

As summarized in Table~\ref{tab:relwork}, existing adaptive DP-FL methods adjust clipping bounds, noise scales, or budget allocations, but typically do not integrate explainability during training.
For instance, standard adaptive clipping methods~\cite{andrew2021differentially} dynamically tune the clipping threshold $C$ using privatized quantiles of client update norm distributions.
However, they focus strictly on update-scale optimization and remain entirely agnostic to model explainability.
Among explainability-aware approaches, the most closely related is Adaptive Gradient Protection (AGP)~\cite{li2023balancing}, an open-loop method that uses Grad-CAM channel importance weights to decide \emph{where} noise lands spatially via binary channel masking while keeping per-round noise magnitude fixed.
In contrast, XCal-FL is a closed-loop method that dynamically regulates \emph{how much} noise (via the noise multiplier $\sigma$) to inject at each step based on a composite signal that (i) causally tests whether the saliency map is decision-relevant via logit drop under masking, (ii) measures decision-boundary geometry via counterfactual margin, and (iii) tracks spatial attention coherence.
Thus, because AGP controls spatial noise distribution while XCal-FL dynamically scales temporal noise magnitude, the two strategies operate on orthogonal axes and are complementary.
Furthermore, XCal-FL's closed-loop formulation enables analyzing how cumulative privacy budget translates into joint gains in accuracy and explanation fidelity, a dimension underexplored in prior work.

\section{Background}
\label{sec:relwork}

\subsection{Differentially Private Federated Learning}\label{sec:dp_fl_background}

Federated Learning (FL) \cite{mcmahan2017communication} is a distributed learning paradigm in which multiple clients, such as hospitals or banks, collaboratively train a shared model while keeping raw data decentralized.
Formally, at communication round $t$, each participating client $i$ initializes its local model parameters $\mathbf{w}_i$ with the global parameters $\mathbf{w}^{(t)}$.
The client minimizes a local loss function $\mathcal{L}$ using stochastic gradient descent.
For a given mini-batch $(x_k, y_k)_{k=1}^B$, the per-sample gradient is computed with respect to the current local model as:
\begin{equation}
g_k = \nabla_{\mathbf{w}_i} \mathcal{L}(\mathbf{w}
_i; x_k, y_k).
\end{equation}
In standard (non-private) FL, the batch gradient $g$ is simply the average of these per-sample gradients, and the parameters are updated iteratively over $E$ local epochs according to $\mathbf{w}_i \leftarrow \mathbf{w}_i - \eta g$, where $\eta$ denotes the learning rate. After completing its local training steps, the resulting parameters form client $i$'s updated model, denoted as $\mathbf{w}_i^{(t+1)}$. The server then aggregates these models to form the new global model:
\begin{equation}\label{eq:fl_param_update}
\mathbf{w}^{(t+1)} = \sum_{i=1}^{N} \frac{n_i}{\sum_{j=1}^{N} n_j}\mathbf{w}_{i}^{(t+1)},
\end{equation}
where $N$ is the number of participating clients and $n_i$ is the local dataset size of client $i$.

To provide formal privacy guarantees, FL is commonly combined with Differential Privacy (DP), which bounds the information that can be inferred about any single individual from the training process.
A randomized mechanism $\mathcal{M}$ is $(\varepsilon,\delta)$-differentially private if, for any neighboring datasets $D$ and $D'$ differing in one record, and for any output set $\mathcal{S}$,
\begin{equation}\label{eq:dp_definition}\Pr[\mathcal{M}(D) \in \mathcal{S}] \leq e^{\varepsilon} \Pr[\mathcal{M}(D') \in \mathcal{S}] + \delta,
\end{equation}
where $\varepsilon$ controls the privacy loss budget and $\delta$ denotes the probability that this guarantee may not hold.
DP can be enforced at different points in the FL pipeline, often yielding different levels of privacy granularity.
Server-side approaches such as DP-FedAvg~\cite{mcmahan2018learning} provide user-level privacy by clipping individual client updates and adding noise to their aggregation on the server.
While this protects against external observers, it requires trust in the server itself to handle raw client updates securely.
Differentially Private client-side DP-SGD~\cite{abadi2016deeplearningwithdp} instead enforces record-level privacy locally: each client clips the per-sample gradients to a fixed norm $C$ and adds calibrated Gaussian noise during local training.
Specifically, instead of using the raw per-sample gradient $g_k$, the gradient is clipped as:
\begin{equation}
\bar{g}_k = \frac{g_k}{\max\left(1,\;\lVert g_k \rVert_2 / C\right)},
\end{equation}
and the noisy aggregated gradient for the batch is computed as:
\begin{equation}\label{eq:dp_sgd_update}
\tilde{g} = \frac{1}{B}\left(\sum_{k=1}^{B} \bar{g}_k + \mathcal{N}(0, \sigma^2 C^2 \mathbf{I})\right),
\end{equation}
The local parameter update step is then performed using this private gradient: $\mathbf{w}_i \leftarrow \mathbf{w}_i - \eta \tilde{g}$.
Here, the noise multiplier $\sigma$ is calibrated such that the cumulative privacy loss across all local training steps satisfies a target $(\varepsilon,\delta)$ privacy budget.
Client-side DP-SGD provides a stronger trust model than standard server-side approaches, as the resulting model updates $\mathbf{w}_i^{(t+1)}$ are differentially private before they ever leave the client device, protecting each client's individual records even against a curious server.
A summary of our notations is given in Table~\ref{tab:hyperparams}.

\subsection{Model Explainability}

Grad-CAM \citep{selvaraju2017grad} is a widely used local attribution method that highlights image regions most influential to a model's prediction, with explanation quality commonly assessed by the confidence drop when these regions are masked.
The Remove And Debias (ROAD) framework~\cite{rong2022consistent} extends this idea into a systematic evaluation of explanation fidelity, i.e., how accurately a saliency map reflects the model's true decision process.
The ROAD score serves as a proxy for model explainability because it provides a computationally efficient and consistent evaluation of saliency-based explanations, such as those provided by Grad-CAMs.
ROAD evaluates explanation fidelity by progressively removing either the most (MoRF) or least (LeRF) relevant pixels first, imputing them with local averages, and measuring the resulting change in the model's output.
Hence, larger performance changes indicate higher explanation fidelity, which is an indication for a higher model explainability.

Formally, let \(f(x)\) denote the model's confidence for input image \(x\), and let 
\(x^{(k)}_{s}\) denote the input obtained by perturbing the top \(k\%\) most ($MoRF$ scheme) and least ($LeRF$ scheme) salient pixels.
For each removal level \(k\), consider the absolute confidence difference for either the \text{MoRF} or the \text{LeRF} scheme $s$ as
\begin{equation}
\Delta_{s}(k) = f\!\left(x^{(k)}_{s}\right) - f(x),
\end{equation}
where \(\Delta_{MoRF}(k)\) corresponds to the model's confidence change after removing the most salient features and \(\Delta_{LeRF}(k)\) to the change after removing the least salient ones.
We follow an approach that progressively removes features across multiple thresholds and averages the resulting confidence changes to produce a single combined ROAD score \cite{jacobgilpytorchcam}, as follows:

\begin{equation}
\mathrm{ROAD}
= \frac{1}{2|K|} \sum_{k \in K} \big( \Delta_{\mathrm{LeRF}}(k) - \Delta_{\mathrm{MoRF}}(k) \big),
\end{equation}
where the set $K$ consists of thresholds corresponding to the top 20\%, 40\%, 60\%, or 80\% relevant pixels.

\subsection{Threat Model}
\label{subsec:threat}

\subsubsection{\textbf{Entity Assumptions.}} We consider a cross-silo FL deployment executed over $T$ communication rounds by two types of participants:
\begin{itemize} \item \textbf{Clients} $i\in\{1,\dots,N\}$, i.e., institutions such as hospitals, each holding a private dataset $\mathcal{D}_i$ of medical images and labels in which a single patient may contribute several images.
Clients have sufficient compute for local training and follow the protocol faithfully.
Each client observes only its own data and local model, together with the global models it receives. 
The attribution maps and the explainability signal derived from them (Section~\ref{sec:design}) are computed and consumed locally and are never transmitted.
\item \textbf{Aggregation server}, which broadcasts $\mathbf{w}^{t-1}$, collects the locally trained parameters $\mathbf{w}^{t}_{i}$ and aggregates them (Eq.~\ref{eq:fl_param_update}).
It holds no training data and never observes raw samples, per-sample gradients or attribution maps. 
\end{itemize}

\subsubsection{\textbf{Communication channels.}} Clients communicate exclusively with the server and never with each other, so no client observes another client's individual update.
All client-server links are assumed to be confidential, integrity protected and mutually authenticated (e.g., TLS with client certificates), as is common in cross-silo federations governed by institutional agreements.
An outside observer can therefore neither read nor modify updates in transit, nor impersonate a participant, and its knowledge is limited to the released global model.
We assume no secure aggregation primitive, because our guarantee is enforced locally before any parameter leaves the client.

\subsubsection{\textbf{Adversary model.}} The primary adversary is an honest-but-curious server that follows the protocol but attempts to infer information about individual training records from the messages it legitimately receives.
It is the strongest passive adversary in our setting, since it observes each client's update $\mathbf{w}^{t}_{i}$ before aggregation and accumulates these observations across rounds.
Its concrete capability is gradient-based reconstruction~\cite{geiping2020inverting,zhu2019deep}, which exploits model updates to recover individual training samples such as chest X-ray scans, exposing patient-identifiable diagnostic information.
An honest-but-curious client, or any party obtaining the released model, sees only aggregated parameters and is therefore strictly weaker, so by the post-processing property of DP~\cite{dwork2014algorithmic} the guarantee established against the server covers it as well.

The unit of protection is a single training record, i.e., neighboring datasets $\mathcal{D}_i$ and $\mathcal{D}'_i$ differ in one record, yielding a record-level $(\varepsilon,\delta)$-DP guarantee.
Clients' defense against reconstruction attempts is client-side DP-SGD, which clips per-sample gradients and injects calibrated Gaussian noise before any parameter leaves the client, with cumulative loss tracked by RDP accounting~\cite{wang2019subsampled}.
Because our method varies the noise multiplier $\sigma$ across mini-batches (Section~\ref{sec:design}), two properties bound what the variation itself reveals.
First, the multiplier never falls below the value that standard, non-adaptive DP-SGD would require for the share of the budget allocated to gradient perturbation, for any input including adversarial or out-of-distribution ones, since the signal driving it is bounded by clipping.
Second, the multiplier is a deterministic function of a privatized signal and of public constants, so any inference the server draws from the noise scale is bounded by the budget share we reserve for releasing that signal, keeping the total within the target $(\varepsilon,\delta)$ guarantee (Section~\ref{subsec:noise_calibration_regularizer} and Appendix~\ref{sec:appendix_privacy_guarantee}).

\subsubsection{\textbf{Limitations.}} Under stronger adversaries, an active server that deviates from the protocol, for instance by sending crafted models to isolate a target client, can degrade utility and the correctness of aggregation, but not the record-level DP guarantee, which holds per release for any input model and composes adaptively~\cite{whitehouse2023fully}.
Malicious clients can corrupt the global model through poisoning, for which robust aggregation is a complementary defense, and a client that inflates its explainability score through crafted inputs can at most lower its noise multiplier to the lower end of the permitted range, which leaves the target guarantee intact and does not affect other clients.
Collusion, whether among clients or with the server, adds nothing beyond what the server already observes.
Finally, the privacy guarantee of our method is record-level, so a patient contributing several images is covered only through group privacy, whose $\varepsilon$ degrades linearly in the number of contributed records~\cite{dwork2014algorithmic}.

\section{XCal-FL: Explainability-driven FL Training with Differentially Private Noise Calibration}
\label{sec:design}

\begin{figure*}[ht]
    \begin{center}
    \centerline{\includegraphics[width=0.80\linewidth]{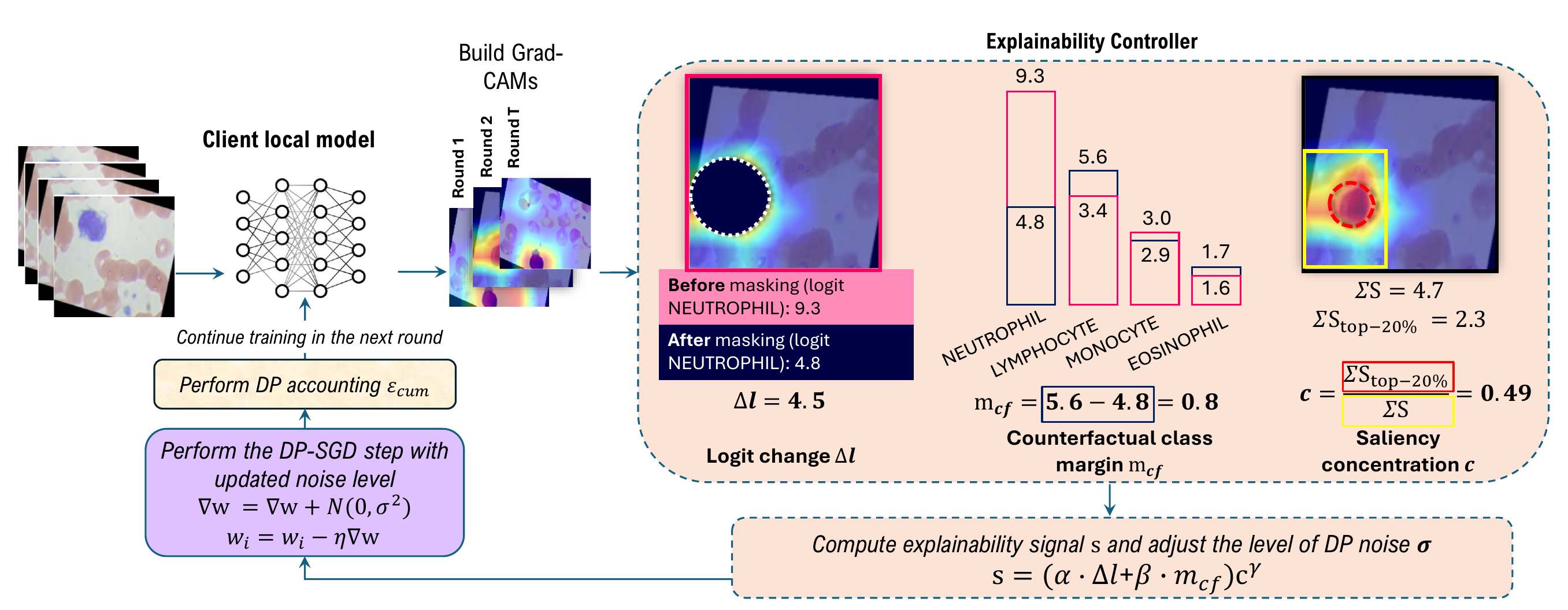}}
    \caption{Overview of FL client's training procedure. The adaptive training process repeats for $T$ iterations for each client.}
    \label{fig:client_training_overview}
    \end{center}
\Description[]{}
\end{figure*}

Current research gaps highlight the need for a FL approach that not only enhances the explainability of model predictions, but also improves how efficiently incurred privacy loss is translated into gains in both performance and explainability.
Hence, we propose a training methodology designed to balance these factors while remaining broadly applicable across diverse imaging tasks.
Our implementation code is available at \url{https://osf.io/xr9nt/overview?view_only=b0fd7b2b88a94578abd06dacb5d19fad}.

\subsection{Main Idea}\label{sec:key_method}

We extend the standard FL training (see server-side Algorithm~\ref{alg:expl_driven_server} in Appendix \ref{sec:appendix_server_side_fedavg_pseudo}) with an explainability-driven client training procedure (Algorithm~\ref{alg:expl_driven_client}) that improves the faithfulness of model explanations while preserving predictive performance and providing DP (Figure~\ref{fig:client_training_overview}).
The key idea is to adapt the amount of DP noise injected into each mini-batch update based on an explainability signal reflected by the model's predictions.
This signal integrates three complementary components, summarized below:

\paragraph{\ding{182} Logit change ($\Delta\ell$)}
    measures the decrease in the predicted-class logit after masking the top-$q$ region of the Grad-CAM heatmap, defined as:
    \begin{equation}
        \Delta \ell = \max\!\big(0,\, \ell_{\hat{y}}(x) - \ell_{\hat{y}}(x^{(q)})\big),
    \end{equation}
    where \(x^{(q)}\) denotes the input masked by the top-$q$ saliency region and \(\hat{y}\) is the highest-scoring predicted class.
    A large logit drop indicates that the masked region is highly influential for the prediction, reflecting stronger explanation faithfulness.

\paragraph{\ding{183} Counterfactual margin (\(m_{\mathrm{cf}}\))}
    quantifies the difference between the highest counterfactual class logit and the masked predicted-class logit:
    \begin{equation}
        m_{\mathrm{cf}} = \max\!\big(0,\, \ell_{\hat{y}_{\mathrm{cf}}}(x^{(q)}) - \ell_{\hat{y}}(x^{(q)})\big),
    \end{equation}
    where \(\hat{y}_{\mathrm{cf}}\) is the highest-scoring counterfactual class.
    A larger margin indicates that the salient region plays a critical role in separating the true class from competing classes, encouraging the model to form sharper and more semantically meaningful decision boundaries.

\paragraph{\ding{184} Saliency concentration (\(c\))}
    measures the fraction of the total saliency sum contained within the top-$q$ region, defined as
    \begin{equation}
        c = \frac{\sum_{u} S_u M_u^{(q)}}{\sum_{u} S_u},
    \end{equation}
    where \(S_u\) denotes the saliency value at spatial location \(u\) and \(M_u^{(q)} \in \{0,1\}\) is the binary mask induced by retaining the top-$q$ fraction of pixels ranked by saliency (averaged over the mini-batch).
    Higher concentration indicates that the model's attention is spatially focused rather than diffuse, yielding more localized and interpretable explanations.

We combine the three components into a composite score
\begin{equation}\label{eq:composite}
s = (\alpha \Delta \ell + \beta m_{\mathrm{cf}}) c^{\gamma},
\end{equation}
where $\alpha$ and $\beta$ weight the decision-level terms $\Delta \ell$ and $m_{\mathrm{cf}}$, and $\gamma$ scales the influence of the saliency concentration $c$.

Each component captures a distinct and complementary aspect of explanation quality: $\Delta\ell$ measures causal influence (whether the salient region affects the prediction), $m_{\text{cf}}$ measures class discriminativeness (whether masking the salient region shifts the decision toward a competing class), and $c$ measures spatial specificity (whether saliency is concentrated rather than diffuse).
Importantly, no single component is sufficient on its own.
For example, a high $\Delta\ell$ with low saliency concentration $c$ indicates that the prediction depends on a broad, unstructured set of features rather than on a compact, interpretable region. Reducing noise in this case would reinforce diffuse attention patterns.
Similarly, a high saliency concentration $c$ with low $\Delta\ell$ means that the model attends to a focused but causally irrelevant region, producing a misleading explanation.
The multiplicative role of $c^\gamma$ in Eq.~(\ref{eq:composite}) ensures that noise is reduced only when causal influence and class discriminativeness are supported by spatially coherent saliency, preventing the reinforcement of any single favorable indicator in isolation.
We note that the explainability signal $s$ is not part of the training objective: the model is trained using the standard cross-entropy loss, and no gradient flows through $s$ back to the model parameters.
This prevents the risks of over-optimization, because the signal $s$ only modulates the noise level applied to the gradient, meaning the model cannot directly optimize for a higher score, without improving its reasoning.

\subsection{XCal-FL Training Procedure}\label{sec:training_procedure}

Our training procedure is detailed in Algorithm \ref{alg:expl_driven_client}.
To satisfy the overall target $(\varepsilon, \delta)$-DP guarantee, the total privacy budget is partitioned into two components: a fixed fraction $\rho$ is allocated to the explainability signal mechanism, $\varepsilon_s = \rho\, \varepsilon$, from which the signal noise scale $\sigma_s$ is calibrated across all training steps, and the remainder is allocated to the gradient perturbations, $\varepsilon_g = (1 - \rho)\,\varepsilon$ (line 3).
Given $\varepsilon_g$ and a failure probability $\delta_g$, we compute a reference noise multiplier $\sigma_{ref}$ via the RDP accountant as the exact value required to achieve $(\varepsilon_g, \delta_g)$-DP under standard DP-SGD with a constant multiplier over all training steps.
To allow dynamic noise scaling in both directions around this reference, we define a symmetric noise band [$\sigma_{min}$, $\sigma_{max}$] = [$(1 - b)\sigma_{ref}, (1 + b)\sigma_{ref}$], where $0 \leq b \leq 1$ (line 4).
Each mini-batch update begins by estimating an explanation map using Grad-CAM to identify the regions that most strongly influence the model's current prediction (lines 5--11).
Concretely, given a Grad-CAM heatmap, we threshold the CAM at the $k$-th largest value and select all pixels with saliency at least this cutoff, with \(k=\lceil qHW\rceil\), where \(H \times W\) denotes the CAM resolution. 
Saliency concentration then measures what fraction of the total saliency contained within this selected region, indicating whether the model's attention is tightly focused or diffused.

\begin{table}[h]
\caption{Main hyper-parameters and notations}
\label{tab:hyperparams}
\begin{center}
\begin{small}
\begin{tabular}{>{\raggedright\arraybackslash}p{0.15\linewidth}>{\raggedright\arraybackslash}p{0.6\linewidth}>{\centering\arraybackslash}p{0.1\linewidth}}
\toprule
\textbf{Notation} & \textbf{Description}  &\textbf{Values}\\
\midrule
$N$ & Number of FL clients &$3, 10, 50$ 
\\
$B$ & Mini-batch size &$32$ 
\\
$T$ & Number of training rounds &$30$ 
\\
$\eta$ & Learning rate &$10^{-3}$ 
\\
$C$ & Gradient clipping norm  &$1.0$ 
\\
$\mathbf{w^{t}}$ & Global model weights at iteration $t$ &
\\
$\mathbf{w^{t}_i}$ & Client's $i$ model weights at iteration $t$ &
\\
$b$ & Noise band expansion factor ($0 \leq b \leq 1$) & $0.2$\\
$\rho$ & Fraction of $\varepsilon$ allocated to composite signal & 0.1 \\
$\sigma_{ref}$ & Reference DP noise multiplier & \\
$\sigma_{min}$ & Lower band limit & \\
$\sigma_{max}$ & Upper band limit & \\
$\sigma_{step}$ & DP noise multiplier at a batch step & \\
$\delta$ & DP failure probability &$10^{-5}$ 
\\
$q$ & Fraction of most salient pixels to mask & $0.2$
\\
$\tau$ & Smoothing coefficient & $0.2$ 
\\
$\Delta\ell$ & Logit change &
\\
$m_{cf}$ & Counterfactual margin &
\\
$c$ & Saliency concentration & 
\\
$\alpha$ & Weight of Logit change $\Delta\ell$ &  
\\
$\beta$ & Weight of Counterfactual margin $m_{\mathrm{cf}}$ & 
\\
$\gamma$ & Weight of Saliency concentration $c$ &  \\
\bottomrule
\end{tabular}
\end{small}
\end{center}
\end{table}

Next, we evaluate how important this salient region is by performing both a forward and a masked-forward pass.
From these runs, we compute three explainability indicators: (i) prediction logit change (line 9), (ii) counterfactual class margin (line 10), and (iii) saliency concentration (line 11).
Together, these components quantify how much the model depends on meaningful, well-localized features.
These three indicators are combined into a unified weighted explainability score $s_{\mathrm{raw}}$ (lines~12--13).
Because this raw score is derived directly from the private mini-batch, it must be privatized before it can be used to modulate the training process. 
We explicitly bound its sensitivity to 1 via clipping and apply the Gaussian mechanism by adding noise scaled by a constant $\sigma_s$ to obtain the differentially private signal $\tilde{s}_{\mathrm{raw}}$ (line~14).
The privatized signal is then smoothed to yield the final explainability score $s$ (line~15), which securely modulates the noise multiplier $\sigma_{\text{step}}$ applied in the model's parameter update for the current mini-batch (line~16).
Before the noisy update is applied, the candidate $\sigma_{step}$ is projected onto the budget-feasible portion of the band (line 17): a multiplier $\sigma$ is feasible if the cumulative RDP cost after one additional step at $\sigma$ still converts to at most $(\varepsilon_g, \delta_g)$.
Since the per-step RDP cost decreases monotonically in $\sigma$, this projection either leaves $\sigma_{step}$ unchanged, raises it to the smallest feasible value, or, when even $\sigma_{max}$ is infeasible, halts noisy updates for the remainder of training (lines 18--19).
The realized cost of the admitted step is then added to the accountant state before the update is applied (line 21).
During each update, we apply standard DP-SGD, where gradients are clipped to limit the influence of any individual sample, and Gaussian noise, scaled according to the clipping norm and the adaptively chosen $\sigma_{\text{step}}$, is added before updating the client's model parameters (lines 22--25).
This ensures that every mini-batch contributes only a controlled and privacy-preserving amount of information to the training process.

The cumulative privacy cost of the gradient updates is tracked by the RDP accountant~\cite{mironov2017renyi,wang2019subsampled}\footnote{Our implementation uses fixed-size mini-batch sampling rather than the Poisson subsampling assumed by RDP accounting. This practice is widely adopted in the literature, though the resulting privacy parameters should be interpreted as approximate~\cite{ponomareva2023dp}.} at every training step, immediately before the noisy update is applied, and the privacy filter described above guarantees that it never exceeds $(\varepsilon_g, \delta_g)$.
Therefore, the composed cost of both mechanisms therefore adheres to the global budget (Appendix~\ref{sec:appendix_privacy_guarantee}).

\begin{algorithm}[ht]
\caption{\textsc{XCal-Client}: Explainability-driven FL Training with Differentially Private Noise Calibration}
\label{alg:expl_driven_client}
\KwIn{Client $i$ training data $\mathcal{D}_i=\{X_i, Y_i\}$, $\mathbf{w}^{t-1}$;
Training parameters: $\eta$, $B$;
DP parameters: $C$, $\sigma_{ref}$, $b$, $\delta$, $\rho$;
Explainability parameters: $q,\ \alpha,\ \beta,\ \gamma,\ \tau$}
\KwOut{Round $t$ client model parameters}
Initialize local model $\mathbf{w}_i^{t} \gets \mathbf{w}^{t-1}$\;
Initialize explainability signal $s \gets 0$\;
Initialize RDP accountant with target budget ($\varepsilon_g$, $\delta_g$)\;
$\sigma_{min} \gets (1 - b)\sigma_{ref}$; $\sigma_{max} \gets (1 + b)\sigma_{ref}$\;
\For{each mini-batch $(\mathbf{x}, \mathbf{y})$ of size $B$ sampled from $\mathcal{D}_i$}{
  \tcp{Compute the explainability controls:}
  Compute model prediction $\tilde{y}$ for $\mathbf{x}$\;
  Compute Grad-CAM $\mathbf{S}$ for $(\mathbf{x}, \tilde{y})$ under $\mathbf{w}_i^{t}$\;
  Construct top-$q$ mask from $\mathbf{S}$\;
  Compute logit change $\Delta \ell$ (before and after masking)\;
  Compute counterfactual margin $m_{\mathrm{cf}}$\;
  Compute saliency concentration $c$\;
  \tcp{Compose the explainability signal:}
  $s_{\mathrm{raw}} \gets (\alpha \cdot \Delta \ell + \beta \cdot m_{\mathrm{cf}})\, c^{\gamma}$\;
  $s_{\mathrm{raw}} \gets \textsc{Clip}(s_{\mathrm{raw}}, 0, 1)$\;
  
  \tcp{Privatize the explainability signal:}
  $\tilde{s}_{\mathrm{raw}} \gets \textsc{Clip}\bigl(s_{\mathrm{raw}} + \mathcal{N}(0, \sigma_s^2), 0, 1\bigr)$\;
  \tcp{Smooth the explainability signal:}
  $s \gets (1 - \tau)\, s + \tau\, \tilde{s}_{\mathrm{raw}}$\;
  \tcp{Adjust DP noise based on Xpl. signal:}
  $\sigma_{\text{step}} \gets \sigma_{\max} - s\,(\sigma_{\max} - \sigma_{\min})$\;
  \tcp{Privacy filter:}
  $\sigma_{\text{step}} \gets$ smallest admissible $\sigma \in [\sigma_{\text{step}}, \sigma_{\max}]$\;
  \If{no such $\sigma$ exists}{
    \Return $(\mathbf{w}_i^{t})$ \tcp*{Budget exhausted}
  }
  \tcp{DP Accounting:}
  Charge the accountant with a step at $\sigma_{\text{step}}$\;
  \tcp{DP-SGD update for this mini-batch:}
  $g_k \gets \nabla_{\mathbf{w}} \mathcal{L}(\mathbf{w}_i^{t}; x_k, y_k)$ for $k = 1,\dots,B$\;
  $g_k \gets \frac{g_k}{\max\bigl(1,\ \lVert g_k \rVert_2 / C\bigr)}$ for all $k$ \tcp*{Clipping to $C$}
  $\tilde{g} \gets \frac{1}{B}\Bigl(\sum_{k=1}^{B} g_k + \mathcal{N}(0,\ \sigma_{\text{step}}^2 C^2 \mathbf{I})\Bigr)$\;
  $\mathbf{w}_i^{t} \gets \mathbf{w}_i^{t} - \eta\, \tilde{g}$\;
  Update cumulative privacy loss $\varepsilon_{\mathrm{cum},i}$\;
}
\KwRet{$\mathbf{w}_i^{(t)}$}
\end{algorithm}

\subsection{Noise Calibration as Signal-Dependent Regularization}
\label{subsec:noise_calibration_regularizer}

Beyond enforcing the privacy budget, the adaptive calculation of $\sigma_{\text{step}}$ (line~15 in Algorithm \ref{alg:expl_driven_client}) admits a learning-theoretic interpretation: it implements a form of signal-dependent regularization on the DP-SGD gradient update, modulating the noise term $\mathcal{N}(0, \sigma^2 C^2 \mathbf{I})$ according to the current quality of the model's explanations.
Under DP-SGD, the noise added to clipped gradients controls the trade-off between privacy and learning fidelity, while inherently acting as a regularizer against overfitting \citep{abadi2016deeplearningwithdp, dwork2015generalization, papernot2021deploy}.
Our adaptive mechanism draws on this insight: when $s$ is high, the model grounds its predictions in coherent evidence, so reducing $\sigma_{\text{step}}$ preserves informative gradient components.
Conversely, when $s$ is low, the gradient signal is unreliable, and higher noise acts as regularization that prevents the model from overfitting to spurious patterns.
Viewed through the privacy budget, this regularization corresponds to a redistribution of privacy expenditure over training: steps with high explanation quality receive less noise and consume budget faster, while steps with low explanation quality receive more noise and conserve budget.
XCal-FL therefore spends the same total privacy budget as the static baseline, but allocates it to the training steps where the gradient signal is most informative, dynamically adjusting the per-step privacy cost while the total privacy loss remains strictly upper-bounded by the formal $(\varepsilon, \delta)$-DP guarantee (Appendix~\ref{sec:appendix_privacy_guarantee}).

An important consequence of this interaction is that predictive performance and explanation fidelity may respond differently to cumulative privacy loss.
While classification depends only on producing the correct output ranking and can tolerate moderate gradient directional corruption \cite{duan2025analyzing}, explanation fidelity relies on fine-grained spatial coherence, which is highly sensitive to perturbations and inherently unstable \cite{zaher2024manifold, ye2025gaussian, woerl2023initialization}.
Since DP-SGD amplifies this stochasticity through per-step noise injection, the relationship between cumulative privacy loss and explanation fidelity need not follow the approximately linear trajectory typically observed for predictive performance.
This asymmetry has a direct practical implication: because the gradient components most relevant to explanation quality may converge at a different rate than those driving predictive accuracy, a single noise schedule optimized for one objective is unlikely to be simultaneously optimal for the other.
XCal-FL addresses this by allowing the noise multiplier to track explanation quality dynamically, effectively decoupling the optimization trajectories of the two objectives within the same privacy budget.
We examine this hypothesis empirically in Section~\ref{sec:privacy_performance_fidelity_tradeoff}.

\section{Empirical Evaluation on Medical Imaging}
\label{sec:eval}

\begin{table*}[t]
\begin{centering}
\caption{F1-score across datasets for our adaptive training (XCal-FL), compared with FL+DP (static DP noise) and AGP \cite{li2023balancing}. Each privacy setting is defined by a target $\varepsilon$ and its corresponding adaptive noise range $[\sigma_{\min}, \sigma_{\max}]$. Results are shown for a FL setting comprised of $N=3$ clients, training an EfficientNet-B0 model with a balanced XCal-FL configuration ($\alpha=\beta=\gamma=1$).}
\label{tab:results_f1_score_per_eps}
\begin{minipage}{\linewidth}
\resizebox{1\linewidth}{!}{%
\begin{tabular}{l|>{\centering\arraybackslash}p{0.08\linewidth}>{\centering\arraybackslash}p{0.08\linewidth}>{\centering\arraybackslash}p{0.08\linewidth}>{\centering\arraybackslash}p{0.08\linewidth}>{\centering\arraybackslash}p{0.08\linewidth}>{\centering\arraybackslash}p{0.08\linewidth}>{\centering\arraybackslash}p{0.08\linewidth}
>{\centering\arraybackslash}p{0.08\linewidth}>{\centering\arraybackslash}p{0.08\linewidth}>{\centering\arraybackslash}p{0.08\linewidth}>{\centering\arraybackslash}p{0.08\linewidth}>{\centering\arraybackslash}p{0.08\linewidth}>{\centering\arraybackslash}p{0.08\linewidth}>{\centering\arraybackslash}p{0.08\linewidth}}\toprule
\textbf{Dataset}&  \multicolumn{2}{c}{No DP} &\multicolumn{3}{c}{$\varepsilon=0.05$ ( $19.68 \le \sigma \le 29.52$)}& \multicolumn{3}{c}{$\varepsilon=0.5$ ( $2.77\le \sigma \le 4.16$)}& \multicolumn{3}{c}{$\varepsilon=5.0$ ( $0.72 \le \sigma \le  1.08$)}& \multicolumn{3}{c}{$\varepsilon=50.0$ ( $0.23\le \sigma \le 0.35$)}\\\cmidrule(lr){2-3}\cmidrule(lr){4-6}\cmidrule(lr){7-9}\cmidrule(lr){10-12}\cmidrule(lr){13-15}
 &  Centr. & FL &FL+DP&AGP& XCal-FL& FL+DP&AGP& XCal-FL& FL+DP& AGP& XCal-FL& FL+DP& AGP&XCal-FL\\ \midrule
Blood Cells& \textbf{0.862}& 0.844 &0.235
&0.285&\textbf{0.290}& 0.654
&0.663&\textbf{0.716}& 0.811
&  0.807&\textbf{0.834}& 0.840 
& 0.849&\textbf{0.856} \\
Chest X-rays & \textbf{0.680}&  0.669 &0.478
&0.465&   \textbf{0.492}& 0.529 
&0.542& \textbf{0.600}& 0.599 
& 0.569& \textbf{0.623}& \textbf{0.658} 
&  0.596& 0.629\\
Melanoma & \textbf{0.896} & 0.777 &0.511&0.464&\textbf{0.530}& 0.693&0.709& \textbf{0.712}& 0.745& 0.759&\textbf{0.761}& 0.777 & 0.855& \textbf{0.886}\\ \bottomrule
\end{tabular}
}

\end{minipage}
\end{centering}
\end{table*}

We evaluated our method on public medical imaging datasets (Table~\ref{tab:datasets} in Appendix \ref{sec:appendix_eval_datasets_desc}), including blood cell classification, pneumonia detection using chest X-rays, and melanoma detection using dermoscopic images, providing a diverse testbed to assess the effects of adaptive DP on performance and explainability.
We used a federated learning setup to train two types of Convolutional Neural Networks (CNNs) with different architectures: EfficientNet-B0 \cite{tan2019efficientnet} with $5.3M$ parameters and ResNet-18 \citep{he2016deep} with $11.7M$ parameters.
We assume a cross-silo environment \citep{yang2019federated} comprising a small number of clients (e.g., organizations such as hospitals) with robust and reliable computational resources.
This assumption aligns with practical cross-silo FL settings, particularly in the medical domain, where clinical image data is primarily held by large organizations, such as hospitals, rather than on edge devices.
Hence, we varied the number of clients participating in the training, with a default of $N=3$ clients, where each client stored the same amount of data (sampled with replacement from our dataset when the dataset size was too small for a large number of clients).
To simulate a non-IID environment, we performed the data splitting among clients while creating both a class label shift and feature shift (covariate shift).
Covariate shift was simulated by applying a linearly varying brightness perturbation across clients, so that each client observes images under systematically different illumination conditions while the underlying label structure is preserved.
We trained each model over $T = 30$ communication rounds (Table~\ref{tab:hyperparams}).

Models were trained under different configurations of explainability components described in Section~\ref{sec:training_procedure}, with gradient privacy controlled via an adaptive noise multiplier $\sigma_{\text{step}}$ and an overall $\delta$ fixed to $10^{-5}$.
We set the fraction of the budget dedicated to the explainability signal to $\rho = 0.10$, meaning $\varepsilon_s = 0.10\varepsilon$ and $\varepsilon_g = 0.90\varepsilon$.
The noise band half-width factor is set to b = 0.2 in our experiments, yielding a symmetric band $\sigma_{step} \in [0.8\sigma_{ref}, 1.2\sigma_{ref}]$.
We set the saliency-based mask fraction to $q=0.2$.
We compare our approach against centralized (non-FL) training, standard FL without DP and AGP \cite{li2023balancing}.
For fair comparison, we ran AGP under our FL setup with a representative Gaussian noise scale averaged over the range for each target $\varepsilon$.
We then fixed the fraction of important channels to $r=0.2$ to match our setting ($q=0.2$).
For each target $\varepsilon$, the static FL+DP baseline uses the constant multiplier that exactly exhausts $\varepsilon$ over $T$ rounds, so that all comparisons between XCal-FL and the baselines are made at matched total privacy expenditure.

After training each FL model, we evaluated predictive performance on the test set using the F1-score, alongside the cumulative privacy loss incurred during training.
The F1-score captures both sensitivity and precision, making it well suited to medical classification tasks that often involve imbalanced class or feature distributions.
Then, we evaluated explainability under adaptive DP noise through quantitative and qualitative analyses.
Quantitatively, we report the relative percentage change in ROAD scores on our test set \citep{rong2022consistent}.
Qualitatively, we generate Grad-CAMs for each model's predicted class and overlay them on test images to examine how increasing noise shifts attention from diagnostic regions to diffuse patterns.
Finally, we quantitatively evaluate the trade-off between privacy, predictive performance and explainability.

\section{Results}
\label{sec:results}

\subsection{Privacy-Utility Trade-off}\label{subsec:results_priv_utility_tradeoff}

The adaptive training procedure consistently improved global FL model performance across different DP configurations compared to static noise training (Table \ref{tab:results_f1_score_per_eps}), addressing the performance aspect of \textbf{RQ1}.
For example, in blood cell type classification, adaptive training achieved an F1-score of 0.716 under a noise range of $2.77 \le \sigma \le 4.16$ (around an $\varepsilon = 0.5$ guarantee), compared to 0.654 with static noise.
Nevertheless, both adaptive and static DP-trained models remained less accurate than the centralized model and the corresponding non-DP FL model, which achieved F1-scores of $0.86$ and $0.84$, respectively.
As expected, we observed an increasing trend in the model's performance as the DP guarantee weakened (lower noise multiplier $\sigma$ leading to higher $\varepsilon$), i.e., for a noise setting of $0.23 \le \sigma \le 0.35$, the performance of all models approached that of their corresponding non-DP FL counterparts.
Our adaptive method consistently outperformed AGP in F1 across most datasets and privacy settings. 
Equally weighting all three signal components (logit change, counterfactual margin, and saliency concentration with $\alpha=\beta=\gamma=1$) yielded higher utility than using any component in isolation.
An ablation study (Section \ref{sec:ablation_study} and Appendix~\ref{sec:appendix_ablation_study}) confirmed that this configuration provides strong performance and explainability across all settings.

To further assess the robustness of XCal-FL, we examined how the number of participating clients and the degree of data heterogeneity affect model performance (Figure~\ref{fig:f1_per_clients_per_iid}).
Under IID partitioning, both EfficientNet-B0 and ResNet-18 improved steadily as the number of clients increased from $N{=}3$ to $N{=}50$.
A similar trend appeared under label shift, a non-IID setting in which class proportions vary across clients: ResNet-18 recovered from $0.70$ at $N{=}3$ to $0.86$ at $N{=}50$, suggesting that aggregating updates from a larger, more label-diverse client pool helps compensate for local class imbalances.
Covariate shift, a more severe non-IID setting in which image features themselves differ across clients, proved considerably more challenging: EfficientNet-B0 plateaued around $0.53$ and ResNet-18 around $0.67$ at $N{=}50$, indicating that when clients differ in their input distributions rather than merely in class proportions, increasing the number of participants yields diminishing returns.
Overall, ResNet-18 benefited more from scaling the client set than EfficientNet-B0 across all heterogeneity conditions, consistently achieving a smaller gap to its centralized upper bound (additional results for all datasets are provided in Appendix \ref{sec:additional_per_num_client_iid_results}).

\begin{figure}[h]
    \centering
    \begin{subfigure}{1\linewidth}
    \centering
    \includegraphics[width=1\textwidth]{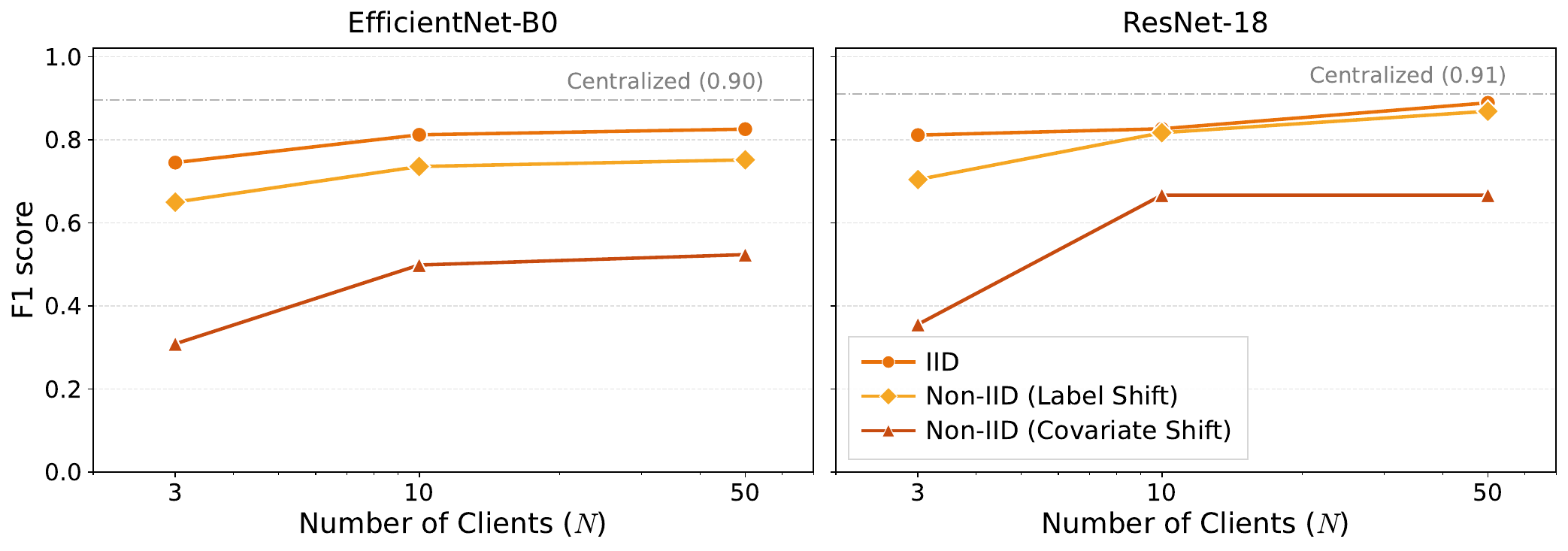}
    \subcaption{Performance (F1 score)}
    \label{fig:f1_per_clients_per_iid}
    \end{subfigure}
    \begin{subfigure}{1\linewidth}
    \centering
    \includegraphics[width=1\textwidth]{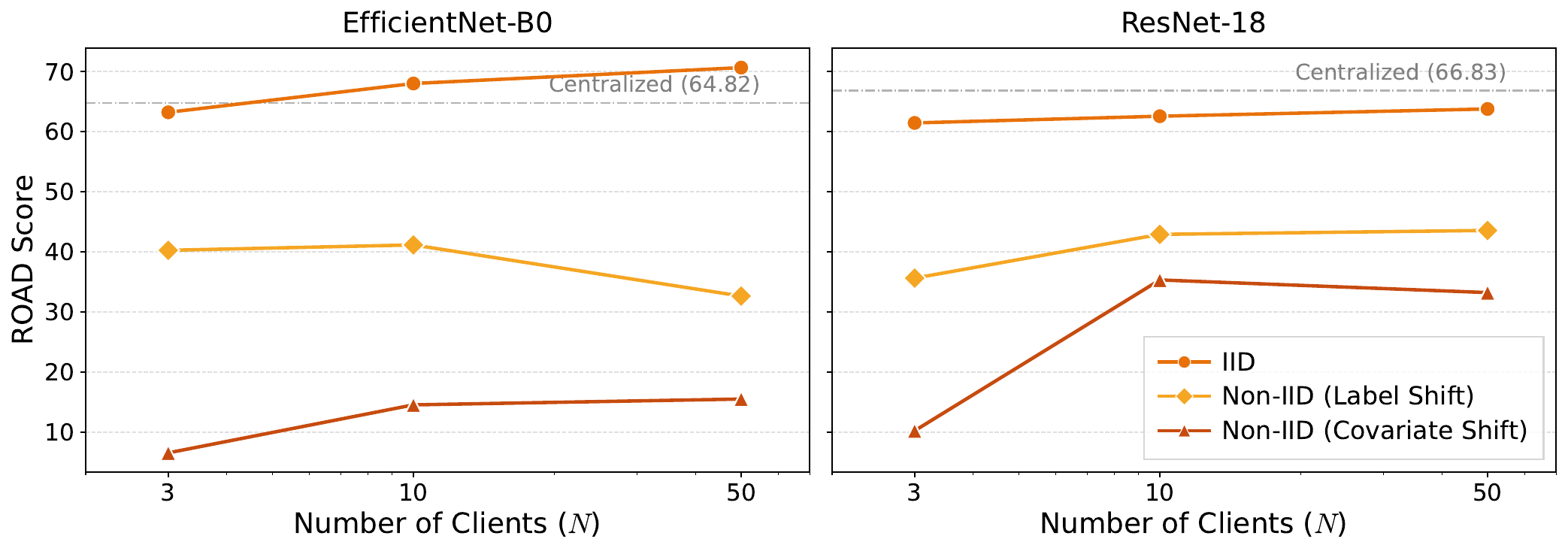}
    \subcaption{Explainability (ROAD score)}
    \label{fig:road_per_clients_per_iid}
    \end{subfigure}
    \caption{F1 and ROAD scores of the global models trained on the Melanoma dataset under our adaptive DP method (XCal-FL) as a function of the number of federated clients under IID and non-IID data distribution settings ($\varepsilon=5$).
    \label{fig:per_clients_per_iid}}
    \Description[]{}
\end{figure}

\begin{table*}[h]
\begin{centering}
\caption{ROAD score (explanation fidelity), averaged over relevance thresholds (top 20\%, 40\%, 60\%, and 80\% pixels), for global models across datasets using XCal-FL, compared with FL+DP (static DP noise) and AGP \cite{li2023balancing}. Each privacy setting is defined by a target $\varepsilon$ and its corresponding adaptive noise range $[\sigma_{\min}, \sigma_{\max}]$. Results are shown for a FL setting comprised of $N=3$ clients, training an EfficientNet-B0 model with a balanced XCal-FL configuration ($\alpha=\beta=\gamma=1$).}
\label{tab:results_explainability_road_per_eps}
\begin{minipage}{\linewidth}
\resizebox{1\linewidth}{!}{%
\begin{tabular}{l|>{\centering\arraybackslash}p{0.08\linewidth}>{\centering\arraybackslash}p{0.08\linewidth}>{\centering\arraybackslash}p{0.08\linewidth}>{\centering\arraybackslash}p{0.08\linewidth}>{\centering\arraybackslash}p{0.08\linewidth}>{\centering\arraybackslash}p{0.08\linewidth}>{\centering\arraybackslash}p{0.08\linewidth}
>{\centering\arraybackslash}p{0.08\linewidth}>{\centering\arraybackslash}p{0.08\linewidth}>{\centering\arraybackslash}p{0.08\linewidth}>{\centering\arraybackslash}p{0.08\linewidth}>{\centering\arraybackslash}p{0.08\linewidth}>{\centering\arraybackslash}p{0.08\linewidth}>{\centering\arraybackslash}p{0.08\linewidth}}
\toprule
\textbf{Dataset}&  \multicolumn{2}{c}{No DP} &\multicolumn{3}{c}{$\varepsilon=0.05$ ( $19.68\le \sigma \le 29.52$)}& \multicolumn{3}{c}{$\varepsilon=0.5$ ( $2.77\le \sigma \le 4.16$)}& \multicolumn{3}{c}{$\varepsilon=5.0$ ( $0.72 \le \sigma \le 1.08$)}& \multicolumn{3}{c}{$\varepsilon=50.0$ ( $0.23\le \sigma \le 0.35$)}\\\cmidrule(lr){2-3}\cmidrule(lr){4-6}\cmidrule(lr){7-9}\cmidrule(lr){10-12}\cmidrule(lr){13-15}
 &  Centr. & FL &FL+DP &AGP &  XCal-FL& FL+DP &AGP & XCal-FL& FL+DP & AGP & XCal-FL& FL+DP & AGP &XCal-FL\\ \midrule
Blood Cells& 202.955&  \textbf{283.255}&25.938
&17.564&\textbf{35.626}&15.561
& 13.456& \textbf{80.568}&86.346
& \textbf{116.614}& 93.803& 113.193 
& 139.758& \textbf{264.896} \\
Chest X-rays & \textbf{407.588}& 187.114 &3.427
&3.536& \textbf{5.758} & 5.742
&\textbf{9.206}&7.229& 9.206& 12.093
& \textbf{13.925}& 4.913 
& \textbf{25.439}& 9.130\\
Melanoma & \textbf{64.819}&  24.921&0.474&2.296&\textbf{8.012}& 31.272&32.610& \textbf{42.661}& 54.056 & 61.072&\textbf{63.220} & 81.462& 101.739 & \textbf{113.956}  \\
\bottomrule
\end{tabular}
}
\end{minipage}
\end{centering}
\end{table*}

\subsection{Model Explainability}\label{subsec:results_explainability}

\begin{figure}
    \centering
    \begin{subfigure}[t]{1\linewidth}
        \centering
        \includegraphics[width=1\linewidth]{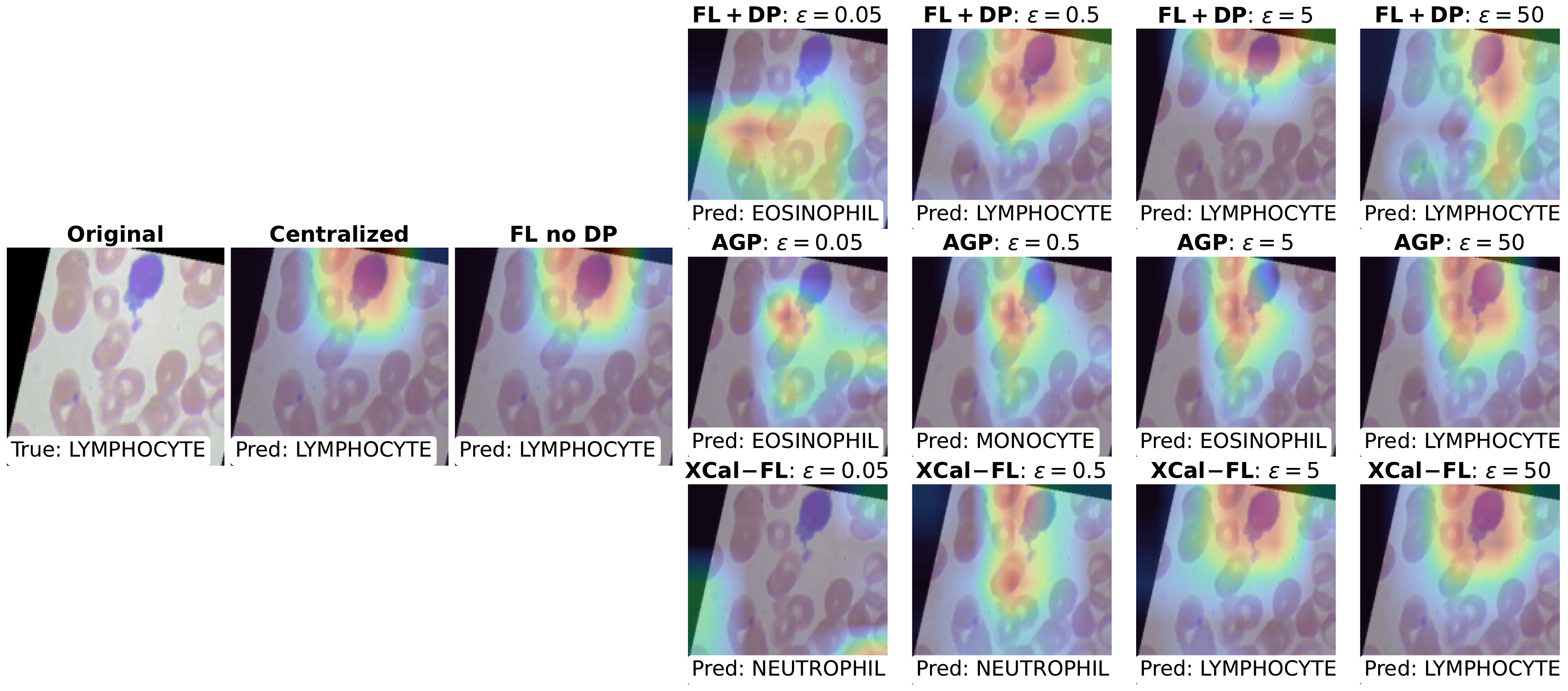}
        \subcaption{Blood cell type detection (class: Lymphocyte)}
        \label{fig:visual_cam_bc}
    \end{subfigure}
    \hfill
    \begin{subfigure}[t]{1\linewidth}
        \centering
        \includegraphics[width=1\linewidth]{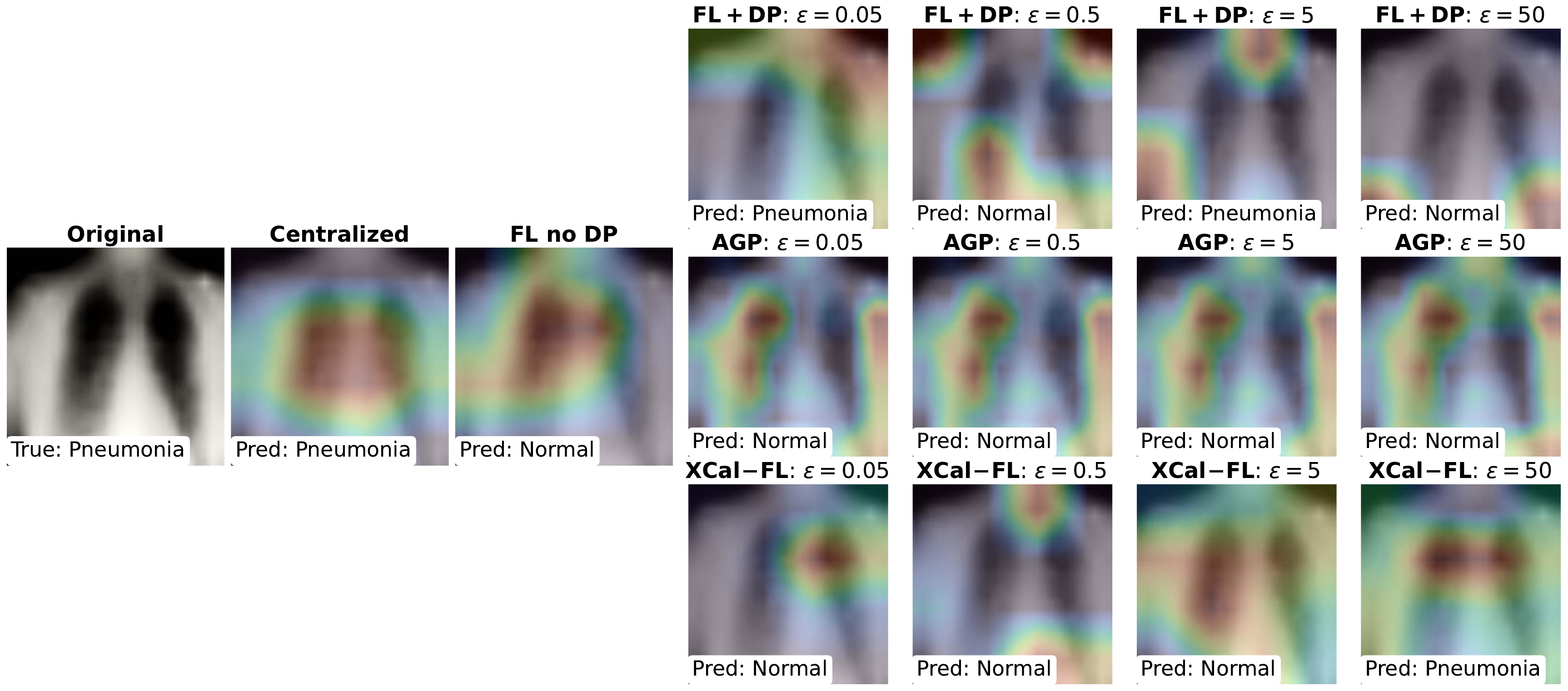}
        \subcaption{Pneumonia detection using Chest X-Rays (class: Pneumonia)}
        \label{fig:visual_cam_xr}
    \end{subfigure}
        \hfill
    \begin{subfigure}[t]{1\linewidth}
        \centering
        \includegraphics[width=1\linewidth]{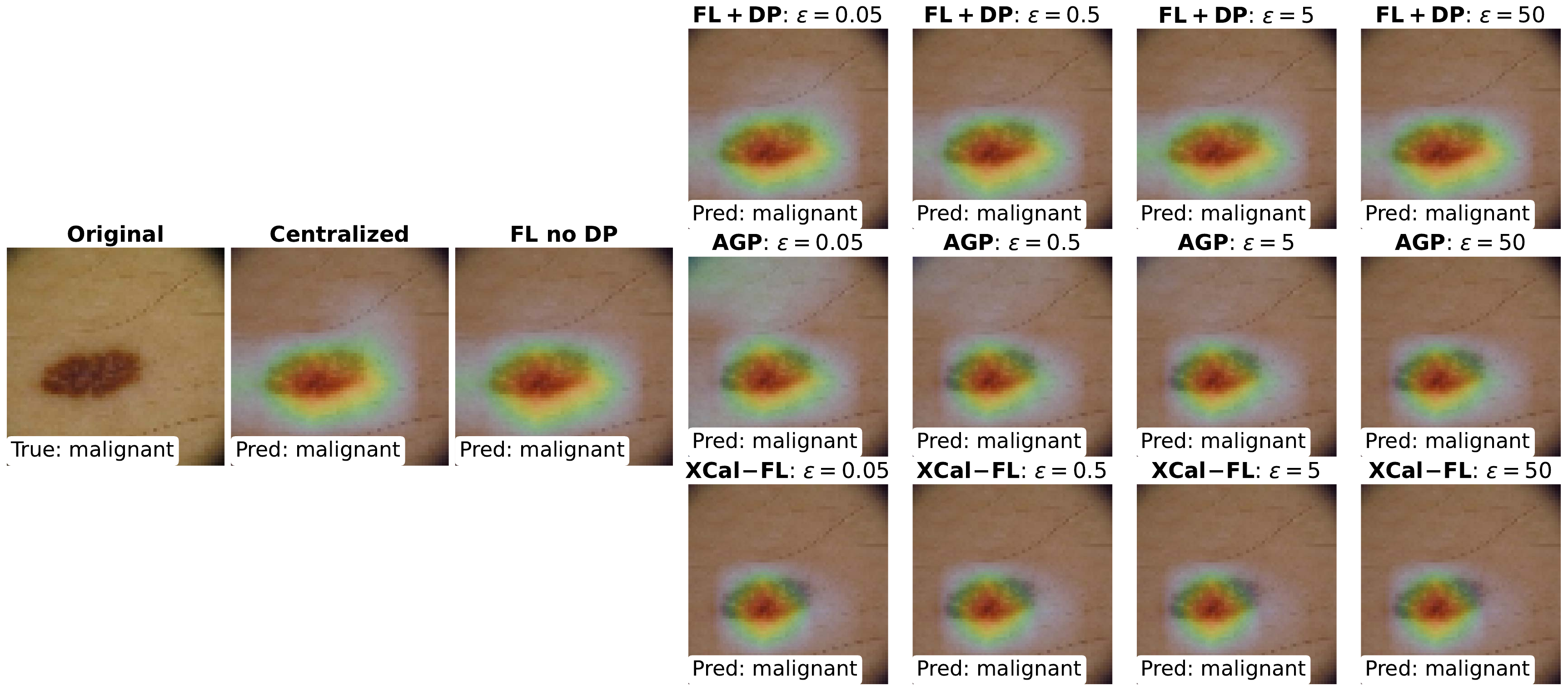}
        \subcaption{Melanoma detection using skin scans (class: Malignant)}
        \label{fig:visual_cam_ml}
    \end{subfigure}
    \caption{Qualitative analysis of model explainability using Grad-CAM visualizations for representative test samples across datasets and DP guarantees ($\alpha=\beta=\gamma=1$). CAMs are generated from the final convolutional layer of the global FL model for the predicted class, with red regions indicating stronger influence.
    }
    \label{fig:fl_cams_dynamic_static_comp}
    \Description[]{}
\end{figure}

Our adaptive FL training procedure led to a larger average drop in model confidence compared to static training, addressing the explainability aspect of \textbf{RQ1}.
A higher drop indicates that masking a small set of salient regions causes a stronger reduction in predicted confidence, reflecting more causal and informative saliency.
Using the ROAD metric averaged across multiple relevance percentiles, we observed a steady increase over our dynamic training, reaching final scores of 80.57\%, 7.23\% and 42.66\% for the blood cell type, pneumonia detection and melanoma classification tasks, respectively (for noise multiplier $2.77 \le \sigma \le 4.16$).
In particular, the blood cell model exhibited a confidence reduction of over 80\% when the most salient regions were removed, indicating a strong causal dependence of its predictions on the identified explanatory features.
In contrast, the static training model showed much lower sensitivity with only a 15.56\%, 5.74\% and 31.27\% confidence drops under the same noise setting for the blood cell type, pneumonia detection and melanoma classification tasks, respectively.
This suggests that the predictions of static training models were only weakly affected by the removal of salient evidence and thus less dependent on meaningful explanatory regions.
These results show that explainability-driven training strengthens the model's reliance on meaningful salient regions, because perturbing these areas leads to a larger and more systematic change in the model's predictions.
Similarly to predictive performance, equally weighting all three signal components ($\alpha=\beta=\gamma=1$) yielded stronger explainability than using any component in isolation (section \ref{sec:ablation_study} and Appendix \ref{sec:appendix_ablation_study}).
For the Chest X-ray pneumonia task, AGP achieved a higher average ROAD score than our method, which we attribute to its explicit suppression of noise on highly salient channels that better preserves gradient-based explanations in this more challenging setting (which is trained on grayscale images).

Furthermore, we observed a clear relationship between the level of DP noise and model explainability, measured by the ROAD metric (Table \ref{tab:results_explainability_road_per_eps}).
As the level of noise increases, gradient directions become less stable and saliency maps less coherent, leading to lower ROAD scores as the model's attention becomes more dispersed across the image.
Conversely, as the DP guarantee weakens (less noise is injected to the model's gradients), saliency maps become sharper and the ROAD metric improves.
This is visually illustrated in Fig. \ref{fig:fl_cams_dynamic_static_comp}, which compares the CAMs of a single test positive image (positive to the existence of a medical issue) in each of our evaluated datasets across varying DP levels, compared to the two baselines.
From visual inspection of the highlighted CAMs, we observed that our adaptive approach produced more concentrated regions, providing a clearer and more focused explanation.
This effect was most pronounced in blood cell classification, where the CAM became smaller and more focused as the amount of injected noise decreased (increasing $\varepsilon$).

In terms of the effect of the FL client set size and their data heterogeneity, as shown in Figure~\ref{fig:road_per_clients_per_iid}, ROAD explainability scores followed a consistent hierarchy across both models: IID yielded the highest scores, followed by label shift, with covariate shift producing the lowest.
Increasing the number of clients had a small positive effect under IID and label shift, consistent with the regularizing benefit of aggregating diverse client updates.
Notably, ResNet-18 exhibited a smaller gap between distribution types, with its ROAD score under covariate shift reaching approximately 35\% compared to under 15\% for EfficientNet-B0, suggesting greater robustness to visual perturbations in non-IID conditions.

\begin{figure*}
    \centering
    \begin{subfigure}{0.321\linewidth}
        \centering
        \includegraphics[width=1\textwidth]{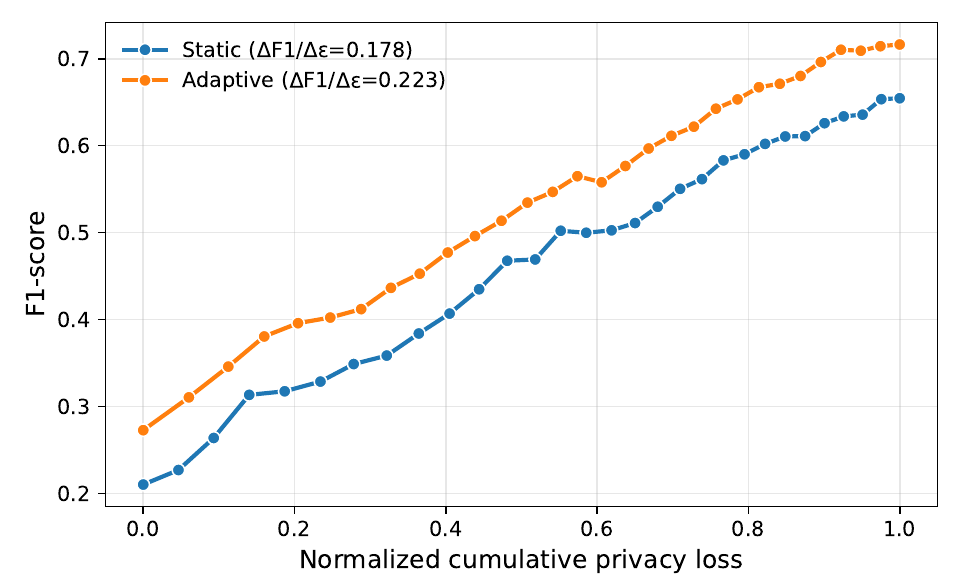}
        \subcaption{Blood cells}
        \label{fig:bc_priv_f1_efficiency}
    \end{subfigure}%
    \begin{subfigure}{0.336\linewidth}
    \centering
    \includegraphics[width=1\textwidth]{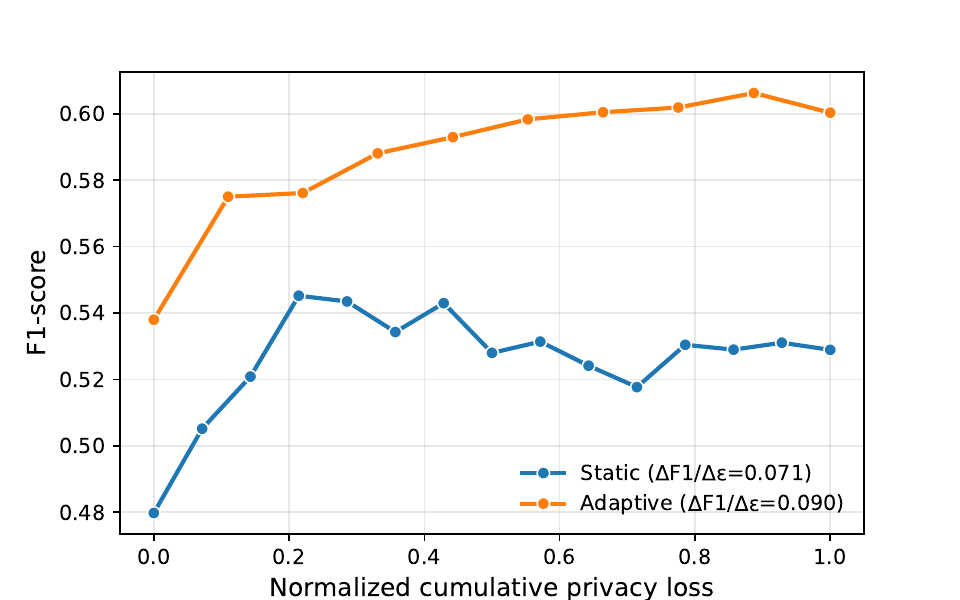}
    \subcaption{Chest X-rays}
    \label{fig:xr_priv_f1_efficiency}
    \end{subfigure}%
    \begin{subfigure}{0.338\linewidth}
    \centering
    \includegraphics[width=1\textwidth]{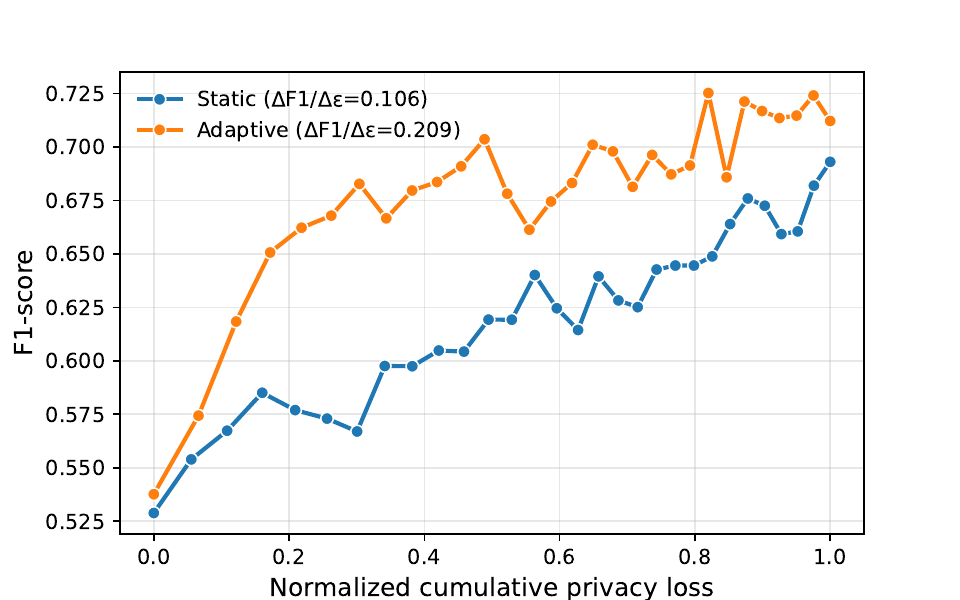}
    \subcaption{Melanoma}
    \label{fig:ml_priv_f1_efficiency}
    \end{subfigure}
    \caption{Privacy-utility efficiency of global FL models across datasets, comparing FL+DP training with \textcolor{MidnightBlue}{static} DP noise and our \textcolor{BurntOrange}{adaptive} XCal-FL training ($\varepsilon = 0.5$ with $\alpha=\beta=\gamma=1$). Efficiency is defined as the ratio of F1 score to the cumulative privacy loss.}
    \label{fig:priv_efficiency}
    \Description[]{}
\end{figure*}

\begin{figure*}
    \centering
    \begin{subfigure}{0.31\linewidth}
        \centering
        \includegraphics[width=1\linewidth]{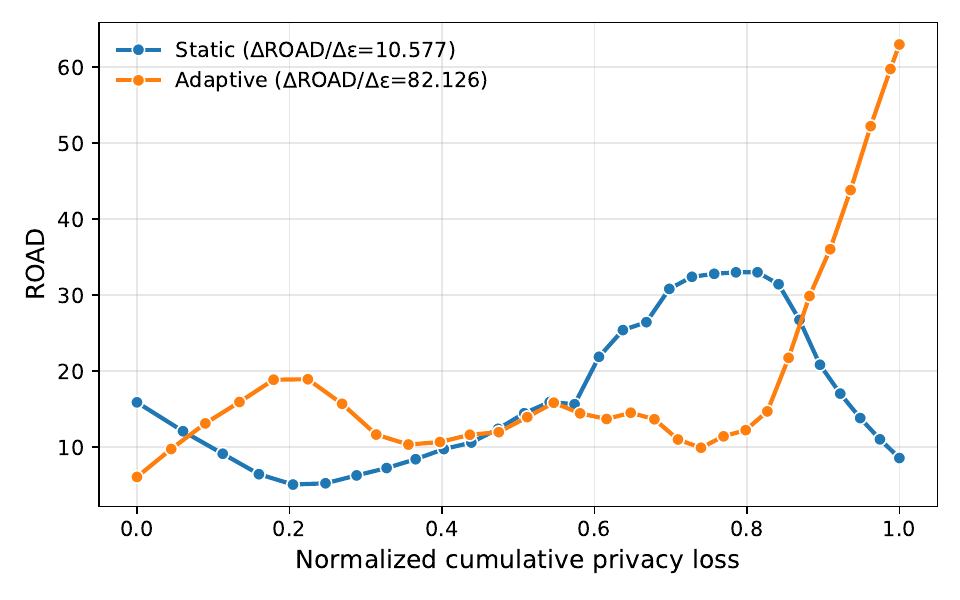}
        \subcaption{Blood cells}
        \label{fig:bc_priv_expl_efficiency}
    \end{subfigure}%
    \begin{subfigure}{0.315\linewidth}
    \centering
    \includegraphics[width=1\linewidth]{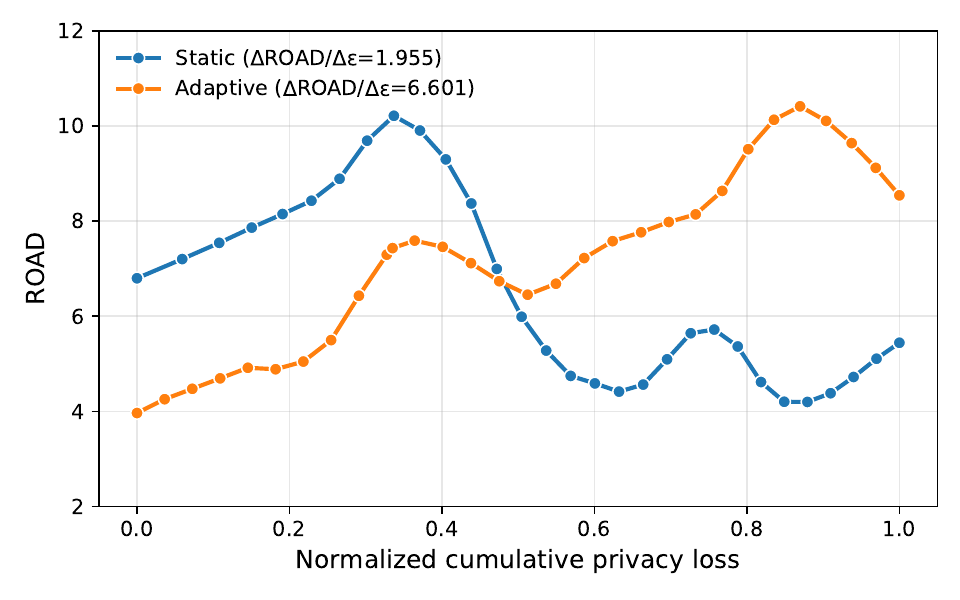}
    \subcaption{Chest X-rays}
    \label{fig:xr_priv_expl_efficiency}
    \end{subfigure}%
    \begin{subfigure}{0.310\linewidth}
    \centering
    \includegraphics[width=1\linewidth]{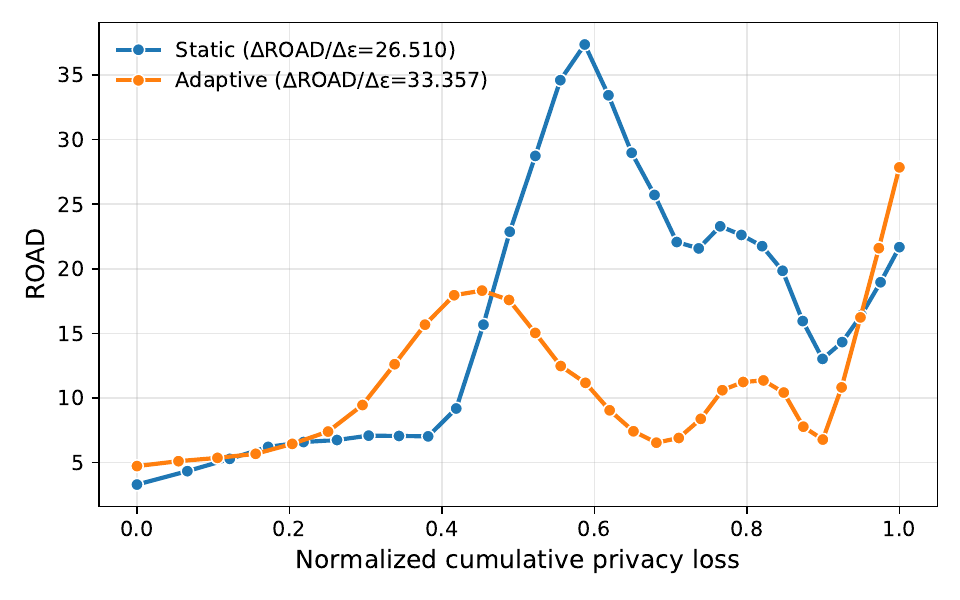}
    \subcaption{Melanoma}
    \label{fig:ml_priv_expl_efficiency}
    \end{subfigure}
    \caption{Privacy-explainability efficiency for global FL models across datasets, comparing FL+DP training with \textcolor{MidnightBlue}{static} DP noise and our \textcolor{BurntOrange}{adaptive} XCal-FL training ($\varepsilon = 0.5$ with $\alpha=\beta=\gamma=1$). ROAD curves were smoothed using LOWESS smoothing. Efficiency is defined as the ratio of ROAD score to the cumulative privacy loss.}
    \label{fig:priv_expl_efficiency}
    \Description[]{}
\end{figure*}

\subsection{Trade-off Analysis Between Privacy, Predictive Performance and Explainability}\label{sec:privacy_performance_fidelity_tradeoff}

In the following sections, we addressed \textbf{RQ2} and \textbf{RQ3} by measuring the privacy-utility and privacy-explainability efficiency as the gain in performance or explanation fidelity ($\Delta F1$ or $\Delta ROAD$) per unit increase in normalized cumulative privacy loss ($\Delta \varepsilon$).

\subsubsection{Privacy-Utility Efficiency}

Figure~\ref{fig:priv_efficiency} summarizes how well each model configuration converts privacy expenditure into predictive gains.
Across all datasets, the adaptive noise injection strategy consistently produced more favorable efficiency trends compared to the static baseline.
For example, in the blood cell classification task (Fig.~\ref{fig:bc_priv_f1_efficiency}), the adaptive model achieved a substantially larger utility gain, increasing the F1-score from $0.20$ to $0.72$ by the last training round.
This corresponds to a utility-to-privacy-loss ratio ($\Delta F1/\Delta\varepsilon$) of $0.223$ for the adaptive configuration compared with $0.178$ for the static one, confirming the higher privacy-utility efficiency achieved under adaptive noise calibration.
While both ratios are below 1, this is expected in FL with DP, where cumulative privacy loss typically grows faster than performance gains.

The skin melanoma classification task (Fig.~\ref{fig:ml_priv_f1_efficiency}) exhibited the strongest efficiency improvement with our adaptive model, achieving nearly double the efficiency of the static baseline by boosting F1 early in training, when privacy loss is still low ($\Delta F1/\Delta\varepsilon$ ratio of $0.209$ and $0.106$ for the adaptive and static training strategies, respectively).
Moreover, while the blood cell task showed a roughly linear efficiency trend, the pneumonia detection task did not follow a clear linear pattern, suggesting dataset-dependent dynamics in how privacy and performance interact.
Overall, these results demonstrate that adaptive noise injection can translate privacy loss into larger performance gains than static training.

\subsubsection{Privacy-Explainability Efficiency}

We observed a similar pattern when evaluating explainability efficiency (Fig.~\ref{fig:priv_expl_efficiency}).
Note that while the trend is not strictly monotonic across training rounds, the adaptive approach consistently achieved higher efficiency across all three tasks.
For example, in the blood cell classification task (Fig.~\ref{fig:bc_priv_expl_efficiency}), the adaptive model reached a local maximum ROAD score after using only 20\% of its normalized privacy budget, demonstrating that strong explanation quality can be achieved at relatively low privacy cost.
In contrast, the static baseline peaked only after consuming roughly 80\% of its privacy budget.
Numerically, this yielded a $\Delta ROAD/\Delta\varepsilon$ ratio of $82.12$ for the adaptive configuration, compared with only $10.57$ for the static one.
For pneumonia detection (Fig.~\ref{fig:xr_priv_expl_efficiency}), the static model reached its maximum ROAD score early in training and showed no further improvement.
In contrast, the adaptive model continued to increase explanation fidelity throughout training, resulting in a higher privacy-explainability ratio of $\Delta ROAD/\Delta\varepsilon=6.60$.
A similar pattern was observed in melanoma detection (Fig.~\ref{fig:ml_priv_expl_efficiency}), where the adaptive model achieved a steeper early rise in the ROAD score and higher efficiency ($\Delta \text{ROAD}/\Delta \varepsilon=33.35$), while the static model peaked after greater privacy expenditure (ratio of $26.51$).

The magnitude of the efficiency improvement varied across datasets.
The blood cell classification task exhibited the largest privacy-explainability efficiency gain (approximately eightfold), which we attribute to the spatially compact and well-localized nature of blood cell features.
In contrast, the melanoma task showed a more modest gain (1.3$\times$), consistent with the observation that dermoscopic features are more distributed across the image, leading to lower saliency concentration scores and a more conservative noise reduction.
The chest X-ray task showed an intermediate gain of 3.3$\times$ over static DP training, lower than the blood cell task, which we attribute to the fact that grayscale images produce less discriminative Grad-CAM maps, limiting the composite signal's dynamic range.
These dataset-dependent patterns suggest that the efficiency gains of adaptive noise calibration are strongest when the task admits spatially concentrated explanatory features.

\subsubsection{Comparing Performance and Explainability Dynamics}
A closer examination of the temporal dynamics in Figures~\ref{fig:priv_efficiency} and~\ref{fig:priv_expl_efficiency} reveals a structural difference between performance and explainability trajectories.
In the blood cell task (Fig.~\ref{fig:bc_priv_expl_efficiency}), the adaptive model reached a peak ROAD score after consuming approximately 20\% of its normalized privacy budget, after which explanation fidelity plateaued.
In contrast, the corresponding F1-score in Figure~\ref{fig:bc_priv_f1_efficiency} continued to increase steadily throughout training.
This decoupling suggests that the gradient components most relevant to explanation quality converge earlier than those driving predictive accuracy.
For the pneumonia detection task (Fig.~\ref{fig:xr_priv_expl_efficiency}), the static model's ROAD score peaked within the first 30\% of training and then stagnated, whereas the adaptive model sustained steady improvement throughout, indicating that explainability-aware noise calibration prevents the early saturation of explanation quality observed under static noise.
These patterns were consistent across datasets, with the timing of the ROAD peak varying by task but always preceding the point at which F1-score stabilized.

Taken together, the efficiency results reveal a fundamental asymmetry in how privacy loss affects predictive performance and explanation fidelity.
The F1-score curves in Figure~\ref{fig:priv_efficiency} follow approximately monotonic, near-linear trajectories across all three datasets, consistent with the standard assumption in privacy-utility analyses that more privacy budget expenditure yields proportionally more utility.
The ROAD curves in Figure~\ref{fig:priv_expl_efficiency}, however, exhibit qualitatively non-monotonic trajectories that reflect unstable behavior.
This asymmetry implies that the privacy-explainability trade-off cannot be inferred from the privacy-utility trade-off alone, and that optimizing one does not guarantee optimization of the other.
This temporal asymmetry also has a practical implication for privacy-sensitive deployments: if explanation fidelity is a primary objective, training can be terminated earlier than would be optimal for predictive performance alone, thereby saving privacy budget.
This demonstrates empirically that performance and explainability follow fundamentally different dynamics under cumulative differential privacy loss in FL, motivating the need for dedicated privacy-explainability evaluation frameworks.

\section{Ablation Study: Explainability Signal Weights} \label{sec:ablation_study}

To assess the individual and combined contributions of the three components within our explainability-guided noise calibration signal, we conducted an ablation study using seven different weight configurations.
These configurations, summarized in Table~\ref{tab:signal_configs}, consist of isolated component tests (where one of the components is active during training, while the others are ignored), balanced baselines (where all components are activated), relative importance comparisons, and concentration scaling variations.
For each configuration, we report predictive performance (macro F1 score) and explainability quality (ROAD score) across all datasets.

\begin{table*}
\centering
\caption{Configurations of explainability signal component weights used in the ablation study. The weights $\alpha$, $\beta$, and $\gamma$ correspond to the logit change, counterfactual margin, and saliency concentration components, respectively.}
\label{tab:signal_configs}
\begin{tabular}{lcccl}
\toprule
\textbf{Configuration} & \textbf{$\alpha$} & \textbf{$\beta$} & \textbf{$\gamma$} & \textbf{Description} \\
\midrule
\multicolumn{5}{l}{\textit{Isolated Components}} \\
Pure Logit Change & 1.0 & 0.0 & 0.0 & Only the logit change component is active \\
Pure Counterfactual Margin & 0.0 & 1.0 & 0.0 & Only the counterfactual margin component is active \\
\midrule
\multicolumn{5}{l}{\textit{Balanced Baseline}} \\
Balanced & 1.0 & 1.0 & 1.0 & All components equally weighted \\
\midrule
\multicolumn{5}{l}{\textit{Relative Importance}} \\
Logit Change-dominant & 2.0 & 1.0 & 1.0 & Logit change weighted twice as much as counterfactual margin \\
Counterfactual Margin-dominant & 1.0 & 2.0 & 1.0 & Counterfactual margin weighted twice as much as logit change \\
\midrule
\multicolumn{5}{l}{\textit{Concentration Scaling}} \\
Dampened Saliency Conc. & 1.0 & 1.0 & 0.5 & Balanced components with square-root saliency concentration \\
Amplified Saliency Conc. & 1.0 & 1.0 & 2.0 & Balanced components with squared saliency concentration \\
\bottomrule
\end{tabular}
\end{table*}

\begin{figure}[h]
    \centering
    \begin{subfigure}{1\linewidth}
    \centering
    \includegraphics[width=1\textwidth]{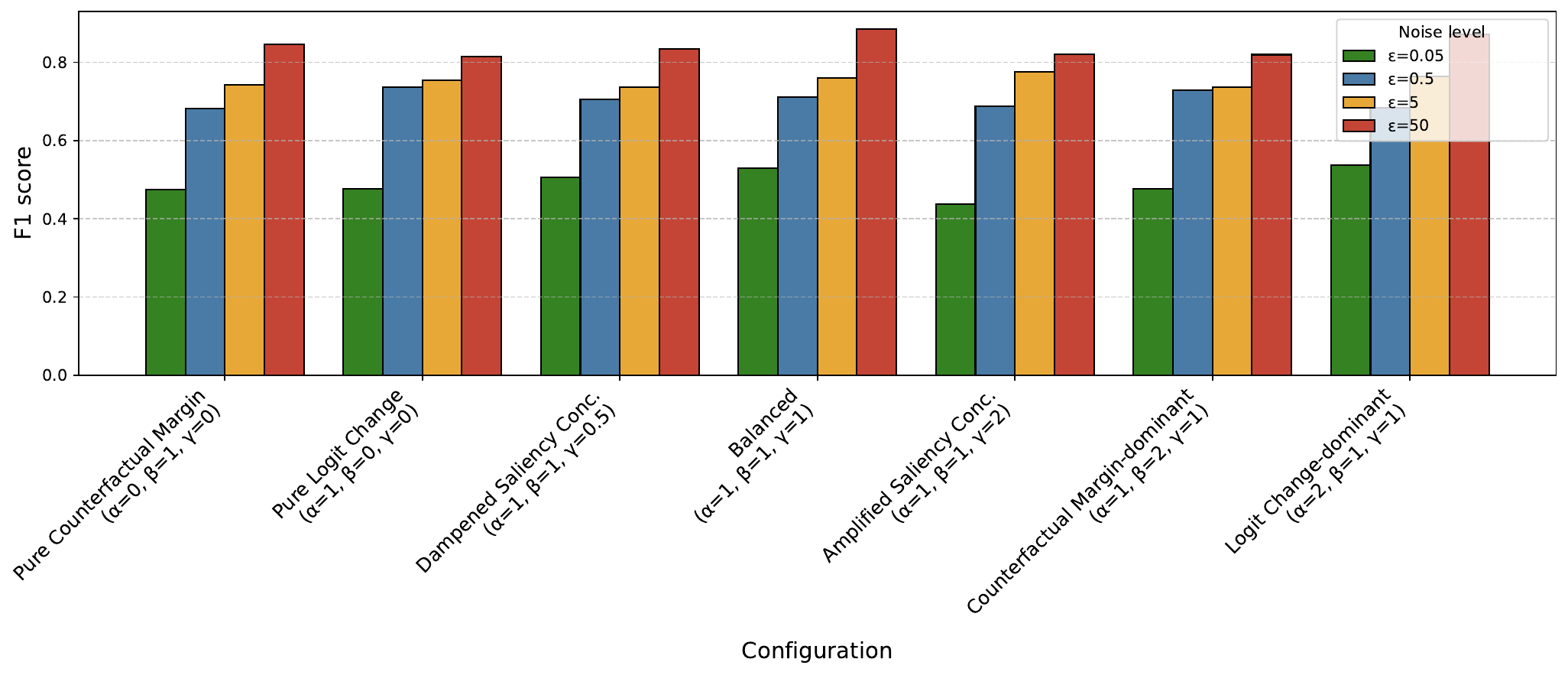}
    \subcaption{Performance (F1 score)}
    \label{fig:ml_road_conf_ablation}
    \end{subfigure}
    \begin{subfigure}{1\linewidth}
    \centering
    \includegraphics[width=1\textwidth]{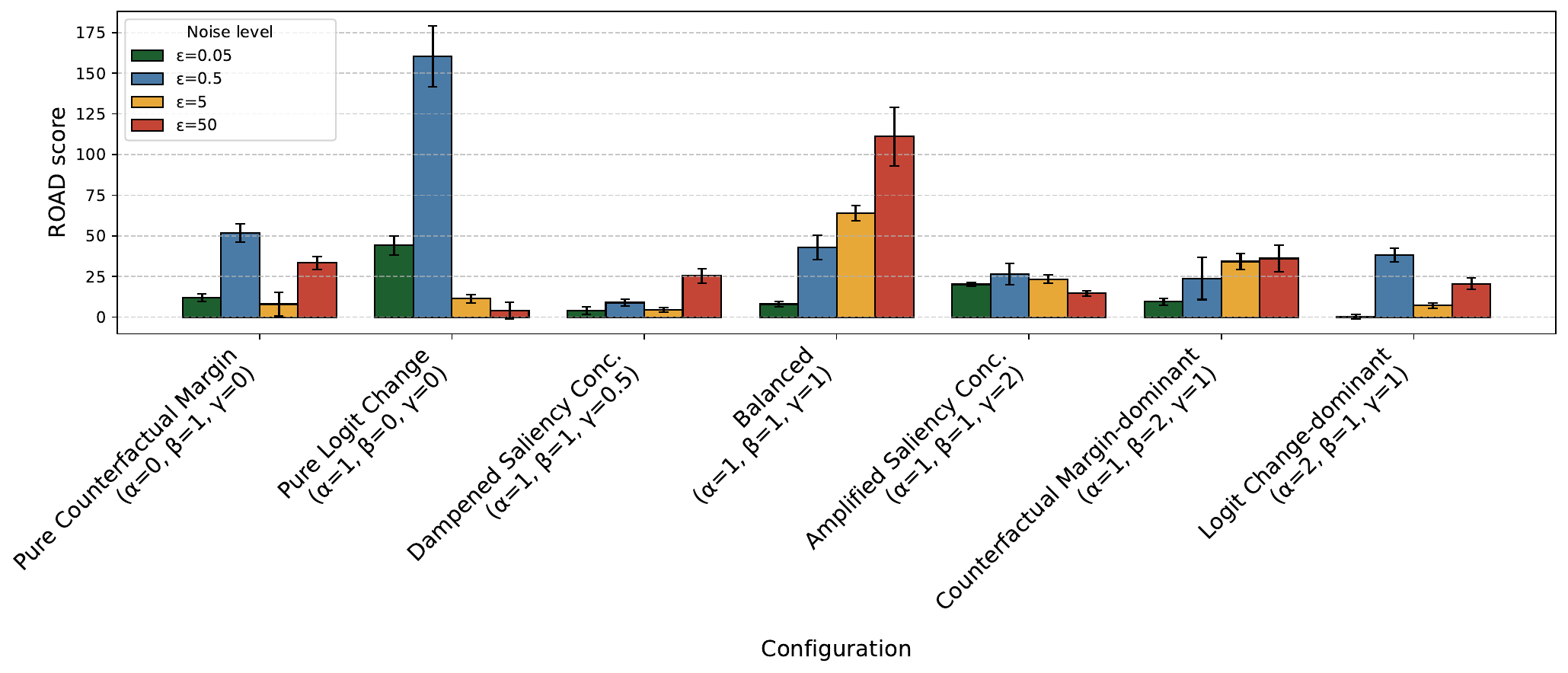}
    \subcaption{Explainability (ROAD score)}
    \label{fig:ml_f1_conf_ablation}
    \end{subfigure}
    \caption{Ablation analysis of the effect of signal component weights on predictive performance (F1 score) and explainability (ROAD score) across privacy guarantee levels ($\varepsilon$) for the Melanoma dataset. Each configuration varies the weights of the logit change ($\alpha$), counterfactual margin ($\beta$), and saliency concentration ($\gamma$) components.
    \label{fig:ml_f1_confs_ablation_analysis}}
    \Description[]{}
\end{figure}

\subsection{Predictive Performance}\label{sec:ablation_study_f1}

Figure~\ref{fig:ml_f1_conf_ablation} presents the F1 scores across all configurations and noise levels for the Melanoma dataset.
All configurations converged to comparable F1 scores ($0.75$--$0.88$) at relaxed privacy levels ($\varepsilon \geq 5$), indicating robustness to moderate variations in component weights.
Under stricter privacy ($\varepsilon \leq 0.5$), the \textit{Balanced} configuration ($\alpha=\beta=\gamma=1$) demonstrated the strongest performance, while the \textit{Amplified Saliency Concentration} variant ($\gamma=2$) showed degraded performance at $\varepsilon=0.05$, suggesting that overly aggressive concentration scaling may be detrimental under high-noise conditions.
We attribute the robustness of the \textit{Balanced} configuration to the complementary nature of the three signal components, which together provide a comprehensive characterization of each sample's explainability quality, enabling finer-grained noise calibration that preserves predictive performance even under strict privacy constraints.
Hence, relying on a single component, by contrast, captures only a partial view of the model’s reasoning, leading to suboptimal noise allocation and reduced utility.
Results for the remaining datasets, which exhibit consistent trends, are provided in Figure \ref{fig:f1_confs_ablation_analysis} in Appendix~\ref{sec:appendix_ablation_study}.

\subsection{Explainability}\label{sec:ablation_study_expl}

Figure~\ref{fig:ml_road_conf_ablation} presents the ROAD scores across configurations and noise levels for the Melanoma dataset.
Explainability differences across configurations were more pronounced than for predictive performance, highlighting the importance of component selection for explainability quality.
The \textit{Balanced} configuration achieved the highest ROAD scores at most noise levels, with particularly strong explainability at $\varepsilon=50$ (ROAD change of $110\%$).
Interestingly, the \textit{Pure Logit Change} configuration ($\alpha=1, \beta=0, \gamma=0$) showed competitive explainability at $\varepsilon=0.5$, but this did not generalize across other noise levels.
The concentration scaling variants show intermediate performance, while the \textit{Counterfactual Margin-dominant} configuration ($\alpha=1, \beta=2, \gamma=1$) exhibited lower and more variable ROAD scores.
These findings support our design choice of incorporating all three components ($\Delta\ell$, $m_{\mathrm{cf}}$, and $c$) rather than relying on a single signal.
We note that the optimal configuration in terms of its effect on explanation fidelity can be dataset-dependent, as discussed in Appendix~\ref{sec:appendix_ablation_study}.

\section{Execution Runtime Complexity Analysis}\label{sec:time_complexity_analysis}

We analyzed the execution runtime complexity of XCal-FL with respect to our used hardware (a machine with a 64GB of RAM and a NVIDIA's RTX 2080 Ti GPU) across all three datasets and our two evaluated model architectures.
A standard DP-SGD training step costs approximately $3F + O(P)$ per mini-batch, where $F$ is the cost of a single training-mode forward pass and $P$ is the total number of parameters \citep{abadi2016deeplearningwithdp,epoch2021backwardforwardFLOPratio}.
XCal-FL introduces a modest overhead by adding an explainability-driven signal computation per mini-batch, which consists of two inference-mode forward passes (one on the original input and one on the masked input) and subsequent saliency mask operations.
The privacy filter adds one RDP accountant query per mini-batch, whose cost is negligible relative to the forward passes.
Empirically, across our evaluated architectures and datasets, this translates to a small absolute per-round increase in wall-clock time for each client compared to standard DP-SGD (typically 1--3 seconds).
Crucially, this computational overhead is purely local and does not increase the payload size or communication costs, which remain the primary bottleneck in cross-silo FL environments.
Given the simultaneous gains in accuracy and explainability, this moderate computational overhead represents a favorable trade-off for privacy-sensitive applications, such as clinical imaging.
A more detailed wall-clock benchmarks across all configurations are provided in Appendix \ref{sec:appendix_time_complexity_analysis}.

\section{Discussion and Limitations}
\label{sec:discussion}

This work introduced XCal-FL, an adaptive noise control mechanism for DP-FL that calibrates noise using a composite signal combining explainability- and decision-related measures.
The three signal components offer complementary contributions: the logit change preserves learning from informative salient regions, the counterfactual margin sharpens decision boundaries, and the saliency concentration favors compact attention patterns.
Together, these components enable finer-grained noise calibration than any single component alone, as confirmed by our ablation study (Section~\ref{sec:ablation_study}), where the balanced configuration ($\alpha = \beta = \gamma = 1$) consistently achieved the strongest or most robust results across datasets and privacy levels.
This approach is motivated by prior findings that DP mechanisms with similar nominal privacy guarantees can yield different downstream behavior~\citep{kaissis2024beyond, ezzeddine2024differential}, and extends recent work on explainability-aware adaptive noise allocation in FL~\citep{li2023balancing}.

In terms of privacy preservation, the calculated noise multiplier $\sigma_{\text{step}}$ in our proposed method is always bounded within the symmetric band $[\sigma_{\min}, \sigma_{\max}]$, and the formal ($\varepsilon$, $\delta$)-DP guarantee is enforced by a R{\'e}nyi accountant \cite{wang2019subsampled} that charges the accountant the realized cost of every step and prevents any update that would exceed the allocated budget (Appendix \ref{sec:appendix_privacy_guarantee}).
XCal-FL therefore provides exactly the same formal guarantee as standard DP-SGD calibrated to the same ($\varepsilon$, $\delta$), while redistributing the privacy expenditure across training steps.
Our adaptive noise calibration redistributes noise across training steps more efficiently than in previous methods, such as the AGP \citep{li2023balancing}, translating a given privacy budget into larger gains in both utility and explainability (Figures \ref{fig:priv_efficiency} and \ref{fig:priv_expl_efficiency}).
These results align with prior findings that careful noise handling can improve the privacy-utility trade-off in DP-FL~\citep{li2024clients}, and extend them to the privacy-explainability dimension.
XCal-FL also consistently outperformed AGP~\citep{li2023balancing}, yielding over 10\% higher predictive performance and more than 75\% improvements in explanation fidelity on average, demonstrating that dynamic noise magnitude calibration provides advantages beyond static channel-level noise suppression.
Our empirical results further validate the signal-dependent regularization interpretation introduced in Section~\ref{subsec:noise_calibration_regularizer}: when the explainability signal is high, reduced noise preserves informative gradient components that reinforce causally meaningful representations, i.e.,  when the signal is low, higher noise acts as regularization that prevents overfitting to spurious or diffuse patterns.
This mechanism functions as an automatic mechanism that tightens regularization on poorly explained samples and loosens it on well-explained ones.

We further observed that while predictive performance improved approximately monotonically with increasing privacy loss, explainability exhibited non-linear dynamics, with ROAD scores peaking early in training before stabilizing (Fig.~\ref{fig:priv_expl_efficiency}) for some configurations.
This behavior is consistent with the known sensitivity of gradient-based attributions to perturbations~\citep{zaher2024manifold} and with theoretical findings that gradient-based saliency maps are inherently unstable under training stochasticity~\citep{ye2025gaussian}, an effect that DP-SGD amplifies through its per-step noise injection.
This finding carries broader implications: standard privacy-utility analyses are insufficient to capture the full impact of DP on model behavior, and explainability requires dedicated evaluation as a distinct dimension of the privacy trade-off.

\paragraph{Applicability Limitations}
XCal-FL is designed for image classification and is applicable to privacy-sensitive computer vision settings beyond clinical imaging, such as autonomous driving, wherever Grad-CAM-based attribution is applicable.
However, the explainability signal is tied to Grad-CAM specifically, so substituting a different attribution method may alter the signal characteristics and require re-tuning the composite score.
More broadly, testing applicability to non-vision domains, where gradient-based saliency maps are not directly available, remains a future research direction.

Beyond the task domain, our cross-silo evaluations across $N = 3, 10, 50$ clients (Figure~\ref{fig:per_clients_per_iid}) show moderate and sometimes inconsistent scaling gains, particularly under covariate shift, suggesting the method is best suited to cross-silo deployments and may require adaptation for large-scale cross-device settings.
Moreover, the two additional inference-mode forward passes per mini-batch incur up to 80\% overhead relative to standard DP-SGD (Section~\ref{sec:time_complexity_analysis}), which, while acceptable in cross-silo settings with sufficient compute, may be challenging for resource-constrained clients.
In terms of privacy granularity, our mechanism provides record-level $(\varepsilon, \delta)$-DP, protecting individual training images.
Since a single patient may contribute multiple images, however, the cumulative information leakage across a patient's records is not formally bounded at the patient level.
Providing user-level DP guarantees would require additional mechanisms such as per-patient gradient grouping and is left as future work.
In addition, while the balanced weight configuration generalizes well across our evaluated datasets, optimal weights can be dataset-dependent (e.g., dampened saliency concentration proved superior for Chest X-rays), requiring empirical tuning during real-world deployment.
Finally, a further consequence of the filter-based design \cite{feldman2021individual} is that a client may exhaust its gradient budget before completing all $T$ rounds if the signal keeps the noise persistently below $\sigma_{ref}$, leading to an early halt of noisy updates.
This behavior makes the total privacy expenditure identical to the static baseline by construction, at the cost of a possibly shorter effective training horizon.
An investigation of other DP accountant is left to future work.

\section{Conclusion}
\label{sec:conclusion}

This paper introduced XCal-FL, an explainability-driven training approach with adaptive Differential Privacy for Federated Learning that calibrates noise based on explanation signals to balance privacy, utility, and explanation fidelity.
Experiments on medical imaging tasks show that XCal-FL consistently outperforms fixed-noise DP-SGD, achieving over 10\% higher performance and more than 75\% higher fidelity of attribution-based explanations than state-of-the-art methods for adaptive DP in FL.
In addition, our approach achieves up to 25\% higher privacy-utility efficiency and up to threefold higher privacy-explainability efficiency, allowing trained FL models to translate the incurred privacy loss into larger gains in both predictive performance and explanation fidelity than standard DP-SGD training with static DP noise.
This supports interpretable and trustworthy decision-making in safety-critical decision-support applications.


\appendix

\begin{acks}

\end{acks}

\begin{ethics}

This research relies exclusively on publicly available datasets obtained from official data releases \citep{bccd_dataset,wang2017chestx,isic2020}.
Although these datasets originate from the medical domain, they were released in anonymized form by their respective providers, and no re-identification was attempted in our experimental setting.
No personally identifiable information was collected, processed, or stored at any stage of this work.
We have complied with the terms of use specified by each data provider (CC-BY-NC for the ISIC Melanoma dataset \citep{isic2020}; MIT license for the Blood Cells dataset \citep{bccd_dataset}; CC0 1.0 Universal license for the NIH Chest X-ray dataset \citep{wang2017chestx}).
The research poses minimal risk, as it involves only secondary analysis of existing public data and does not involve human participants, or interaction with vulnerable populations.
Therefore, ethical approval from an IRB was not required.

\end{ethics}

\begin{openscience}

An anonymous repository containing our code is available at~\url{https://osf.io/xr9nt/overview?view_only=b0fd7b2b88a94578abd06dacb5d19fad}.
The repository includes the implementation of the proposed XCal-FL training procedure, covering both client-side training (implementing explainability-driven FL training with differentially private noise calibration) and server-side aggregation, the two evaluated model architectures, and configuration files for reproducing the reported experiments.
The datasets used in this work are publicly available from their respective providers and are therefore not redistributed in the repository.
\end{openscience}

\begin{ai}

We used generative AI-based tools to revise the text, improve flow, and correct typos, grammatical errors, and awkward phrasing. We have manually verified the text and are responsible for the accuracy, originality, and integrity of it.

\end{ai}

\bibliographystyle{ACM-Reference-Format}
\bibliography{reference}

\section*{Appendix}
\section{FL with Adaptive Explainability-Driven Client Training (Server Side)}
\label{sec:appendix_server_side_fedavg_pseudo}

Algorithm~\ref{alg:expl_driven_server} outlines the server-side procedure of our explainability-driven DP-FL framework, utilizing the notation defined in Table~\ref{tab:hyperparams}, building upon the standard Federated Averaging (FedAvg) protocol \citep{mcmahan2017communication}.
The server is parameterized by three types of parameters, as follows.
\begin{itemize}
    \item \textbf{FL parameters} including the number of global rounds $T$ and the client set $\mathcal{M}$ with $|\mathcal{M}| = N$.
    \item \textbf{Local training parameters} $\Phi_{\text{tr}}$, including the client models' learning rate $\eta$ and the batch size $B$.
    \item \textbf{Explainability parameters} $\Phi_{\text{expl}}$, defined by the proposed method in Section~\ref{sec:key_method}.
    \item \textbf{Differential Privacy parameters} $\Phi_{\text{dp}}$, consisting of the clipping norm $C$, the reference multiplier $\sigma_{ref}$ and the band factor $b$ defining the noise range $[\sigma_{\min}, \sigma_{\max}]$ and the failure probability $\delta$, which quantifies the relaxation of $\varepsilon$-DP and is fixed to a small constant in accordance with standard DP practice \citep{ponomareva2023dp}.
\end{itemize}

At each round $t$, the server samples a subset of $m$ clients $S_t \subset \mathcal{M}$ and distributes the current global model $\mathbf{w}^{t-1}$ to sampled clients.
Upon receiving the model, each client $i \in S_t$ executes the adaptive explainability-driven training locally (based on DP-SGD algorithm \citep{mcmahan2017communication}) described in Section \ref{sec:design}.
We collect the locally trained (and differentially private) model parameters $\mathbf{w}_i^{t}$, and the client's cumulative privacy loss $\varepsilon_{\text{cum}}$ for privacy accounting.
The server then performs standard weighted aggregation to derive the aggregated model parameters $\mathbf{w}^{t}$ and broadcasts this updated global model to all clients for the subsequent round.

\begin{algorithm}[tbh]
\caption{FL with Adaptive Explainability-Driven Client Training (Server Side)}
\label{alg:expl_driven_server}
\KwIn{FL parameters: global rounds $T$; client set $\mathcal{M}$;
Training parameters $\Phi_{\text{tr}} = (\eta,\ B)$;
DP parameters $\Phi_{\text{dp}} = (C, \sigma_{ref},b, \delta,\rho)$;
Explainability parameters $\Phi_{\text{expl}} = (q,\ \alpha,\ \beta,\ \gamma,\ \tau)$}
\KwOut{$\mathbf{w}^{T}$ final global model parameters}
Initialize global model $\mathbf{w}^{0}$\;
\For{$t = 1 \to T$}{
  Sample a subset of clients $S_t \subset \mathcal{M}$ ($|S_t| = m$)\;
  \ForAll{client $i \in S_t$ with $n_i$ samples}{
    $\mathbf{w}_i^{t} \gets \textsc{XCal-Client}\!\bigl(\mathbf{w}^{t-1},\ \mathcal{D}_i,\ \Phi_{\text{tr}}, \Phi_{\text{expl}},\ \Phi_{\text{dp}}\bigr)$\;
  }
  $\mathbf{w}^{t} \gets \frac{1}{\sum_{i \in S_t} n_i}\sum_{i \in S_t} n_i\, \mathbf{w}_i^{t}$ \tcp*{Aggregation}
  Broadcast $\mathbf{w}^{t}$ to all clients\;
}
\end{algorithm}

\section{Privacy Guarantee Under Adaptive Noise Calibration}
\label{sec:appendix_privacy_guarantee}

We demonstrate that XCal-FL satisfies the target $(\varepsilon, \delta)$-DP guarantee under fully adaptive composition~\cite{whitehouse2023fully}, using a privacy filter~\cite{feldman2021individual} to enforce the gradient budget.
The overall budget is partitioned into two components: a budget of $(\varepsilon_s, \delta_s)$ allocated to the computation of the explainability signals across all rounds, and a budget of $(\varepsilon_g, \delta_g)$ allocated to the gradient perturbations, so that $(\varepsilon, \delta) = (\varepsilon_s + \varepsilon_g,\, \delta_s + \delta_g)$.

The reference multiplier $\sigma_{\mathrm{ref}}$ is computed via binary search over the R\'enyi Differential Privacy (RDP) accountant~\cite{mironov2017renyi, wang2019subsampled} as the constant value that exactly satisfies $(\varepsilon_g, \delta_g)$-DP over $T$ training steps at batch size $B$, and the adaptive multiplier is selected from the symmetric band $[\sigma_{\min}, \sigma_{\max}] = [(1-b)\,\sigma_{\mathrm{ref}},\, (1+b)\,\sigma_{\mathrm{ref}}]$ with $0 < b < 1$.
Because $\sigma_{\mathrm{step}}$ may fall below $\sigma_{\mathrm{ref}}$, adherence to the gradient budget is not implied by the band itself.
Instead, each client runs a privacy filter: before each noisy update, the accountant computes the cumulative privacy cost of all gradient releases so far, including one additional step at the candidate $\sigma_{\mathrm{step}}$, and the step is admitted only if this cost remains within $(\varepsilon_g, \delta_g)$; otherwise, $\sigma_{\mathrm{step}}$ is clamped upward to the smallest admissible value in $[\sigma_{\mathrm{step}}, \sigma_{\max}]$, and if no value in the band is admissible, the client halts noisy updates for the remainder of training.
Throughout, privacy costs are tracked under the subsampling assumption discussed in Section~\ref{sec:training_procedure}, which applies identically to XCal-FL and to the static baseline and is orthogonal to the adaptive calibration analyzed here~\cite{ponomareva2023dp}.

We establish the guarantee with respect to the \emph{adversary's view} of the client's training process, namely the full sequence of released outputs: the privatized explainability signals and the noisy gradient updates.
Any quantity available to a party in our threat model, including the local parameters $\mathbf{w}_i^t$ received by the server and the realized multipliers $\sigma_{\mathrm{step}}$ at iteration $t$, is obtained by post-processing this sequence, so by the post-processing property of DP~\cite{dwork2014algorithmic} the guarantee extends to all of them.
 
\begin{proposition}
\label{prop:privacy}
Let the sequence of explainability signals $\tilde{s}$ be computed via an $(\varepsilon_s, \delta_s)$-DP mechanism, and let the gradient noise multipliers$\sigma_{\mathrm{step}} \in [\sigma_{\min}, \sigma_{\max}]$ be selected as a post-processing of $\tilde{s}$, subject to the privacy filter with budget $(\varepsilon_g, \delta_g)$.
Then XCal-FL satisfies an overall $(\varepsilon_s + \varepsilon_g,\, \delta_s + \delta_g)$-DP guarantee.
\end{proposition}

\begin{proof}
\textbf{Step 1 (Signal Privacy).}
The raw explainability signal $s_{\mathrm{raw}}$ is clipped to $[0, 1]$, which bounds its sensitivity to any single training record by $1$. The mechanism $M_1$ adds Gaussian noise $\mathcal{N}(0, \sigma_s^2)$ to the clipped signal. With $\sigma_s$ calibrated by the accountant to the composition of all $T$ signal releases, the sequence of privatized signals $\tilde{s}$ satisfies the allocated $(\varepsilon_s, \delta_s)$-DP budget.

\textbf{Step 2 (Post-Processing and Adaptivity).}
The step-specific multiplier $\sigma_{\mathrm{step}}$ is a deterministic function of the privatized signal $\tilde{s}$, the public band $[\sigma_{\min}, \sigma_{\max}]$, and the state of the privacy filter, which is itself a function of previously released private outputs. By the post-processing property of DP~\cite{dwork2014algorithmic}, computing $\sigma_{\mathrm{step}}$ incurs no additional privacy loss. Moreover,since $\sigma_{\mathrm{step}}$ depends on the training data only through previously released differentially private outputs, its use as the noise parameter of the subsequent gradient mechanism is admissible under adaptive composition~\cite{whitehouse2023fully, feldman2021individual}.

\textbf{Step 3 (Filter Enforcement of the Gradient Budget).}
Let $M_2$ denote the gradient-release mechanism, namely the composition of all noisy gradient updates $\tilde{g}_t$ produced during training, where the $t$-th update is computed with the adaptively chosen multiplier $\sigma_{\mathrm{step}, t}$. Conditioned on all previously released outputs, each admitted update of $M_2$ is a standard DP-SGD step, i.e., a subsampled Gaussian mechanism with a known multiplier $\sigma_{\mathrm{step}}$, whose privacy cost is computed by the accountant~\cite{abadi2016deeplearningwithdp,mironov2017renyi}. Because the multipliers are chosen adaptively, composition theorems for parameters fixed in advance do not directly apply. The theory of privacy filters~\cite{feldman2021individual} establishes that enforcing a fixed budget through the admission, clamping, and halting rule described above yields a valid privacy guarantee for $M_2$: by construction, the realized cumulative privacy cost of all gradient releases never exceeds $\varepsilon_g$ at failure probability $\delta_g$. Such filters have been instantiated for DP-SGD with adaptively chosen noise~\cite{feldman2021individual}, and this enforcement incurs no loss relative to composition with parameters fixed in advance~\cite{whitehouse2023fully}. Note that the guarantee does not depend on the lower band limit $\sigma_{\min}$: the band $[\sigma_{\min}, \sigma_{\max}]$ is a design parameter governing the calibration range, while the budget is enforced solely by the filter. In particular, a client whose signal keeps the noise persistently below $\sigma_{\mathrm{ref}}$ may exhaust its gradient budget before completing all $T$ steps, at which point noisy updates halt; the total privacy expenditure therefore never exceeds that of the static baseline, while the effective training horizon may be shorter.

\textbf{Step 4 (Sequential Composition).}
By the sequential composition property of DP~\cite{dwork2014algorithmic}, composing the $(\varepsilon_s, \delta_s)$-DP signal mechanism $M_1$ with the filtered, adaptively scaled $(\varepsilon_g, \delta_g)$-DP gradient mechanism $M_2$ yields a total privacy loss bounded by $(\varepsilon_s + \varepsilon_g,\, \delta_s + \delta_g)$. Hence, XCal-FL satisfies the target $(\varepsilon, \delta)$-DP guarantee.
\end{proof}

\section{Description of Medical Imaging Datasets}
\label{sec:appendix_eval_datasets_desc}

Table \ref{tab:datasets} lists the datasets that were used for evaluation of our method. The datasets are representative of medical imaging tasks prevalent in real-world clinical use-cases, such as pneumonia detection using chest X-ray scans or melanoma detection using skin lesions scans.

\begin{table}[hb]
\begin{center}
\begin{small}
\caption{Evaluation datasets from the healthcare domain}
\label{tab:datasets}
\begin{tabular}{>{\raggedright\arraybackslash}p{0.29\linewidth}>{\centering\arraybackslash}p{0.13\linewidth}>{\centering\arraybackslash}p{0.1\linewidth}>{\centering\arraybackslash}p{0.1\linewidth}>{\centering\arraybackslash}p{0.18\linewidth}}
\toprule
\textbf{Dataset}& \textbf{Image size}&\textbf{Train size} &\textbf{Test size}&\textbf{Classes}\\
\midrule
Blood cells \cite{bccd_dataset}& 3$\times$150$\times$150& 9,996& 2,487 &Eosinophil; Lymphocyte; Monocyte; Neutrophil\\
NIH Chest X-ray \cite{wang2017chestx}& 3$\times$128$\times$128&78,468&22,433&Normal; Pneumonia\\
ISIC Melanoma \cite{isic2020}& 3$\times$64$\times$64&9,605&1,000&Benign; Malignant\\
\bottomrule
\end{tabular}
\end{small}
\end{center}
\end{table}

\section{Additional Results: Performance and Explainability Across Datasets}\label{sec:additional_per_num_client_iid_results}

Tables~\ref{tab:f1_all} and~\ref{tab:road_all} present the F1 scores and ROAD explanation fidelity scores, respectively, for all evaluated configurations across three medical imaging datasets, two model architectures, and varying numbers of FL clients under IID and non-IID data distributions.

\begin{table*}[ht]
\centering
\caption{F1 scores of global models ($\varepsilon{=}0.5$) across three
  medical imaging datasets for varying FL client numbers, model architectures and client data distribution patterns. XCal-FL models were trained with a balanced XCal-FL configuration ($\alpha=\beta=\gamma=1$).
  \textbf{Bold} marks the higher (better) score between the two DP methods (FL+DP and XCal-FL) per configuration.
  Centralized (non-FL) upper bounds ---
  Melanoma: EfficientNet-B0\,=\,.896, ResNet-18\,=\,.911;
  BloodCells: EfficientNet-B0\,=\,.862, ResNet-18\,=\,.855;
  Chest X-rays: EfficientNet-B0\,=\,.680, ResNet-18\,=\,.647.
}
\label{tab:f1_all}
\footnotesize
\setlength{\tabcolsep}{4.5pt}
\begin{tabular}{@{} ll c ccc ccc ccc @{}}
\toprule
 & & & \multicolumn{3}{c}{\textbf{Melanoma}}&  \multicolumn{3}{c}{\textbf{BloodCells}}&  \multicolumn{3}{c}{\textbf{ChestXray}}\\
\cmidrule(lr){4-6} \cmidrule(lr){7-9} \cmidrule(lr){10-12}
\textbf{Setting} & \textbf{Model} & $\boldsymbol{N}$& FL&FL+DP& XCal-FL&  FL&FL+DP& XCal-FL&  FL&FL+DP& XCal-FL\\
\midrule

  \multirow{6}{*}{IID}
  & EfficientNet-B0&  3  & .777& .693& \textbf{.712}&  .844& .654& \textbf{.716}&   .669& .529&  \textbf{.600}\\
  
  & EfficientNet-B0& 10  & .830&.720 & \textbf{.809}&  .876&.745 & \textbf{.815}&  .634&.535 & \textbf{.689}\\
  
  & EfficientNet-B0& 50  & .818&.670 & \textbf{.822}&  .840&.689 & \textbf{.756}&  .618&.502 & \textbf{.662}\\
  \cmidrule(lr){2-12}
  
  & ResNet-18&  3  & .862&.667& \textbf{.811}&  .861&.740 & \textbf{.809}&  .639&.541 & \textbf{.555}\\
  
  & ResNet-18& 10  & .870&\textbf{.847}& .818&  .862&.733 & \textbf{.802}&  .631&.529 & \textbf{.583} \\
  
  & ResNet-18& 50  & .872&.849& \textbf{.889}&  .835&.685 & \textbf{.751}&  .612&.496 & \textbf{.597}\\

\midrule

  \multirow{6}{*}{\shortstack[l]{Label\\Shift}}
  & EfficientNet-B0&  3  & .783&\textbf{.741}& .650&  .864&.685 & \textbf{.811}&  .638&.538 & \textbf{.545}\\
  
  & EfficientNet-B0& 10  & .796&.685 & \textbf{.736}&  .882&.679 & \textbf{.811}&  .625&.519 & \textbf{.574} \\
  
  & EfficientNet-B0& 50  & .809&.667 & \textbf{.752}&  .845&.625 & \textbf{.752}&  .601&.487 & \textbf{.593}\\
  \cmidrule(lr){2-12}
  
  & ResNet-18&  3  & .886&\textbf{.832}& .705&  .841&.664 & \textbf{.732}&  .633&.531 & \textbf{.538}\\
  
  & ResNet-18& 10  & .872&.774 & \textbf{.807}&  .843&.658 & \textbf{.784}&  .619&.512 & \textbf{.568} \\
  
  & ResNet-18& 50  & .855&.667 & \textbf{.879}&  .822&.616 & \textbf{.791}&  .595&.479 & \textbf{.586}\\

\midrule

  \multirow{6}{*}{\shortstack[l]{Cov.\\Shift}}
  & EfficientNet-B0&  3  & .285&\textbf{.335}& .319&  .876&.733 & \textbf{.750}&  .596&.481 & \textbf{.492}\\
  
  & EfficientNet-B0& 10  & .655&\textbf{.520} & .507&  .883&.724 & \textbf{.804}&  .578&.458 & \textbf{.521} \\
  
  & EfficientNet-B0& 50  & .510&.251 & \textbf{.537}&  .852&.673 & \textbf{.809}&  .553&.427 & \textbf{.543}\\
  \cmidrule(lr){2-12}
  
  & ResNet-18&  3  & .668&\textbf{.661} & .351&  .838&.718 & \textbf{.737}&  .589&.474 & \textbf{.485}\\
  
  & ResNet-18& 10  & .672&\textbf{.667}& .667&  .857&.703 & \textbf{.780}&  .571&.451 & \textbf{.514} \\
  
  & ResNet-18& 50  & .670&\textbf{.667}& .667&  .855&.654 & \textbf{.796}&  .546&.421 & \textbf{.526}\\

\bottomrule
\end{tabular}
\end{table*}

\begin{table*}[ht]
\centering
\caption{%
  ROAD scores (explanation fidelity) of global models ($\varepsilon{=}0.5$) across three
  medical imaging datasets for varying FL client numbers, model architectures and client data distribution patterns.
  XCal-FL models were trained with a balanced XCal-FL configuration ($\alpha=\beta=\gamma=1$).
  \textbf{Bold} marks the higher (better) score between the two DP methods (FL+DP and XCal-FL) per configuration.
  Centralized (non-FL) upper bounds ---
  Melanoma: EfficientNet-B0\,=\,64.819, ResNet-18\,=\,66.830;
  BloodCells: EfficientNet-B0\,=\,202.955, ResNet-18\,=\,200.200;
  Chest X-rays: EfficientNet-B0\,=\,407.588, ResNet-18\,=\,266.563.
}
\label{tab:road_all}
\footnotesize
\setlength{\tabcolsep}{4.5pt}
\begin{tabular}{@{} ll c ccc ccc ccc @{}}
\toprule
 & & &  \multicolumn{3}{c}{\textbf{Melanoma}}&  \multicolumn{3}{c}{\textbf{BloodCells}}&  \multicolumn{3}{c}{\textbf{ChestXray}}\\
\cmidrule(lr){4-6} \cmidrule(lr){7-9} \cmidrule(lr){10-12}
\textbf{Setting} & \textbf{Model} & $\boldsymbol{N}$&  FL&FL+DP& XCal-FL&  FL&FL+DP& XCal-FL&  FL&FL+DP& XCal-FL\\
\midrule
  \multirow{6}{*}{IID}
  & EfficientNet-B0&  3  &  24.921&31.272& \textbf{42.661}&  283.255&15.561& \textbf{80.568}&  187.114&5.742& \textbf{7.229}\\
  
  & EfficientNet-B0& 10  &  26.148&63.584& \textbf{67.215}&  271.842&18.274& \textbf{75.218}&  192.418&6.184& \textbf{6.842}\\
  
  & EfficientNet-B0& 50  &  23.572&69.813& \textbf{70.011}&  258.415&12.584& \textbf{82.415}&  178.215&4.518& \textbf{7.584}\\
  \cmidrule(lr){2-12}
  
  & ResNet-18&  3  &  27.835&\textbf{64.192}& 61.883&  268.418&17.842& \textbf{88.215}&  168.842&6.415& \textbf{8.218}\\
  
  & ResNet-18& 10  &  25.416&\textbf{62.748}& 62.215&  279.584&14.218& \textbf{78.842}&  175.218&5.218& \textbf{7.842}\\
  
  & ResNet-18& 50  &  28.192&60.584& \textbf{63.673}&  254.218&19.584& \textbf{85.418}&  162.584&\textbf{6.842}& 6.518\\
\midrule
  \multirow{6}{*}{\shortstack[l]{Label\\Shift}}
  & EfficientNet-B0&  3  &  19.284&24.715& \textbf{40.143}&  248.572&11.284& \textbf{62.415}&  155.218&4.218& \textbf{5.842}\\
  
  & EfficientNet-B0& 10  &  21.538&26.192& \textbf{41.427}&  255.218&13.842& \textbf{58.842}&  162.842&3.518& \textbf{6.215}\\
  
  & EfficientNet-B0& 50  &  18.415&22.874& \textbf{33.218}&  238.415&9.518& \textbf{65.218}&  148.415&\textbf{4.842}& 4.518\\
  \cmidrule(lr){2-12}
  
  & ResNet-18&  3  &  23.142&29.458& \textbf{36.675}&  242.184&13.518& \textbf{72.842}&  148.584&5.218& \textbf{6.518}\\
  
  & ResNet-18& 10  &  20.875&27.318& \textbf{42.284}&  258.415&10.842& \textbf{68.215}&  155.215&4.415& \textbf{5.842}\\
  
  & ResNet-18& 50  &  22.418&25.192& \textbf{37.548}&  235.842&14.218& \textbf{64.518}&  142.842&\textbf{5.518}& 5.218\\
\midrule
  \multirow{6}{*}{\shortstack[l]{Cov.\\Shift}}
  & EfficientNet-B0&  3  &  14.582&\textbf{8.274}& 6.375&  215.842&8.418& \textbf{48.572}&  118.584&2.842& \textbf{4.518}\\
  
  & EfficientNet-B0& 10  &  15.218&\textbf{19.842}& 15.138&  228.415&10.215& \textbf{52.184}&  125.218&3.218& \textbf{3.842}\\
  
  & EfficientNet-B0& 50  &  12.874&16.528& \textbf{17.842}&  198.584&7.218& \textbf{44.815}&  108.415&2.218& \textbf{4.215}\\
  \cmidrule(lr){2-12}
  
  & ResNet-18&  3  &  18.415&\textbf{22.584}& 10.218&  222.518&9.842& \textbf{55.218}&  128.215&3.415& \textbf{3.518}\\
  
  & ResNet-18& 10  &  16.842&24.128& \textbf{35.584}&  218.184&11.518& \textbf{48.842}&  122.842&2.842& \textbf{4.584}\\
  
  & ResNet-18& 50  &  17.218&20.415& \textbf{31.472}&  208.415&8.184& \textbf{52.415}&  115.518&3.842& \textbf{5.218}\\
\bottomrule
\end{tabular}
\end{table*}

\section{Ablation Study: Additional Results} \label{sec:appendix_ablation_study}

The signal score $s$ is computed as $s = (\alpha \Delta \ell + \beta m_{\mathrm{cf}}) c^{\gamma}$, where $\Delta \ell$ measures the logit change after masking salient regions, $m_{\mathrm{cf}}$ quantifies the counterfactual margin, and $c$ represents the saliency concentration (see section \ref{sec:key_method}).
To assess the individual and combined contributions of the three components in our explainability-guided noise calibration signal, we conducted an ablation study across seven configurations of the component weights ($\alpha$, $\beta$, $\gamma$), where $\alpha$ weights the logit change, $\beta$ weights the counterfactual margin, and $\gamma$ scales the saliency concentration.
These configurations, summarized in Table~\ref{tab:signal_configs}, consist of isolated component tests (where one of the components is active during training, while the others are ignored), balanced baselines (where all components are activated), relative importance comparisons, and concentration scaling variations.
For each configuration, we report predictive performance (macro F1 score) and explainability quality (ROAD score) across all datasets.

\subsection{Predictive Performance}\label{sec:appendix_ablation_study_f1}

Figure~\ref{fig:f1_confs_ablation_analysis} presents the F1 scores across all configurations and noise levels for the Blood Cells and Chest X-ray datasets, respectively.
For the Blood Cells dataset (Figure~\ref{fig:bc_f1_conf_ablation}), F1 scores were consistent across configurations, particularly at higher privacy budgets ($\varepsilon \geq 5$), where all configurations achieved F1 scores between $0.80$ and $0.85$.
Differences become more apparent under stricter privacy constraints ($\varepsilon \leq 0.5$), where the balanced configuration ($\alpha=\beta=\gamma=1$) achieved the highest F1 score at $\varepsilon=0.5$ ($0.72$), outperforming both isolated component configurations and the concentration scaling variants.
The \textit{Counterfactual Margin-dominant} configuration ($\beta=2, \alpha=\gamma=1$) showed slightly reduced performance at moderate noise levels compared to the \textit{Balanced} baseline.

For the Chest X-ray dataset (Figure~\ref{fig:xr_f1_conf_ablation}), F1 scores followed a similar pattern, with all configurations achieving comparable performance at relaxed privacy levels ($\varepsilon \geq 5$), ranging between $0.55$ and $0.63$.
The \textit{Balanced} configuration achieved the highest F1 scores across most noise levels, reaching a score above $0.61$ at both $\varepsilon=5$ and $\varepsilon=50$.
Under stricter privacy ($\varepsilon=0.05$), the \textit{Balanced} configuration maintained competitive performance ($0.49$), while the \textit{Logit Change-dominant} configuration ($\alpha=2, \beta=\gamma=1$) showed the lowest F1 score ($0.43$), suggesting that overweighting the logit change component may harm predictive performance under high-noise conditions for this dataset.

These results indicate that predictive performance is robust to moderate variations in component weights across all three datasets, with the \textit{Balanced} configuration consistently performing well across all noise levels.
We attribute this robustness to the complementary nature of the three signal components: the logit change captures how much the model's prediction depends on salient regions, the counterfactual margin quantifies the decision boundary proximity, and the saliency concentration measures how focused the model's attention is.
When combined, these components provide a more comprehensive characterization of each sample's explainability quality, enabling finer-grained noise calibration that preserves predictive performance even under strict privacy constraints.
Relying on a single component, by contrast, captures only a partial view of the model's reasoning, leading to suboptimal noise allocation and reduced utility.

\begin{figure*}[ht]
    \centering
    \begin{subfigure}{1\linewidth}
        \centering
        \includegraphics[width=0.80\textwidth]{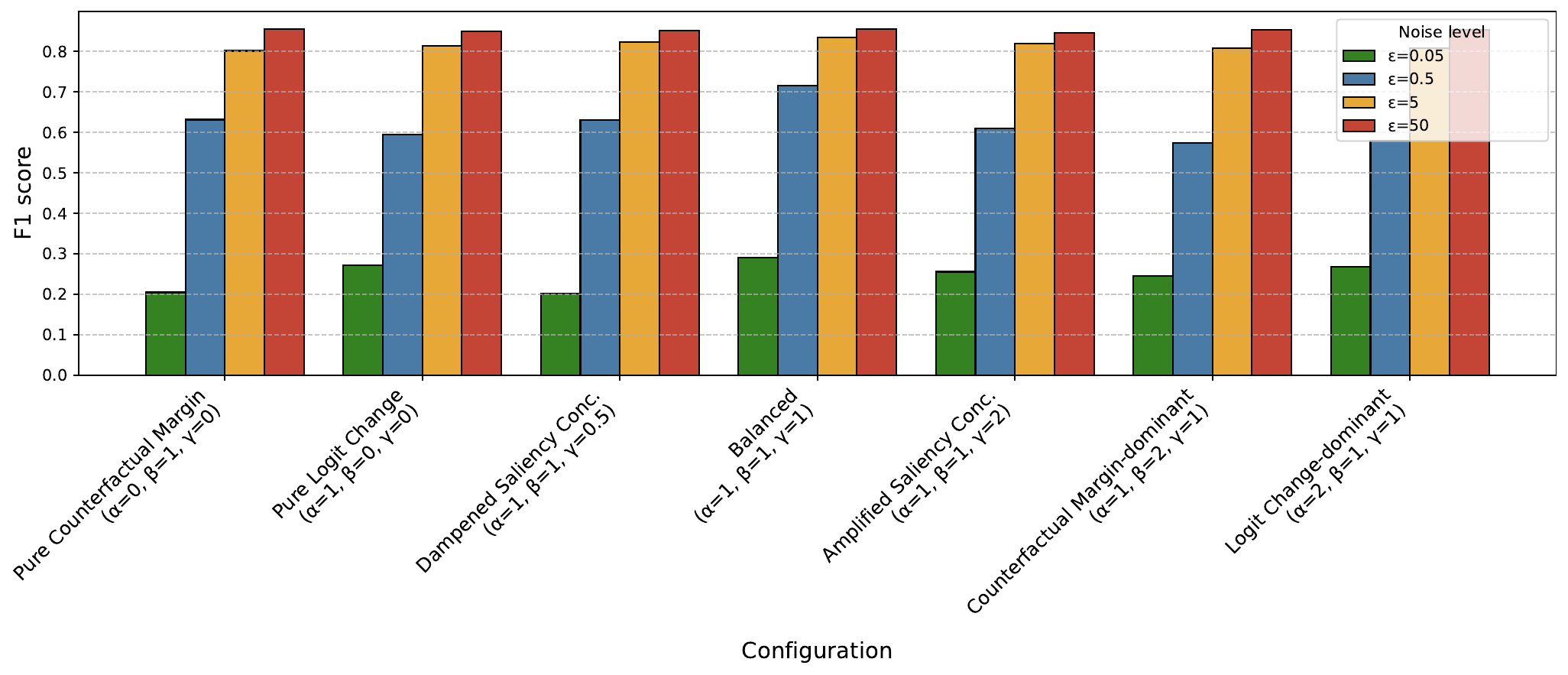}
        \subcaption{Blood cells}
        \label{fig:bc_f1_conf_ablation}
    \end{subfigure}
    \begin{subfigure}{1\linewidth}
    \centering
    \includegraphics[width=0.80\textwidth]{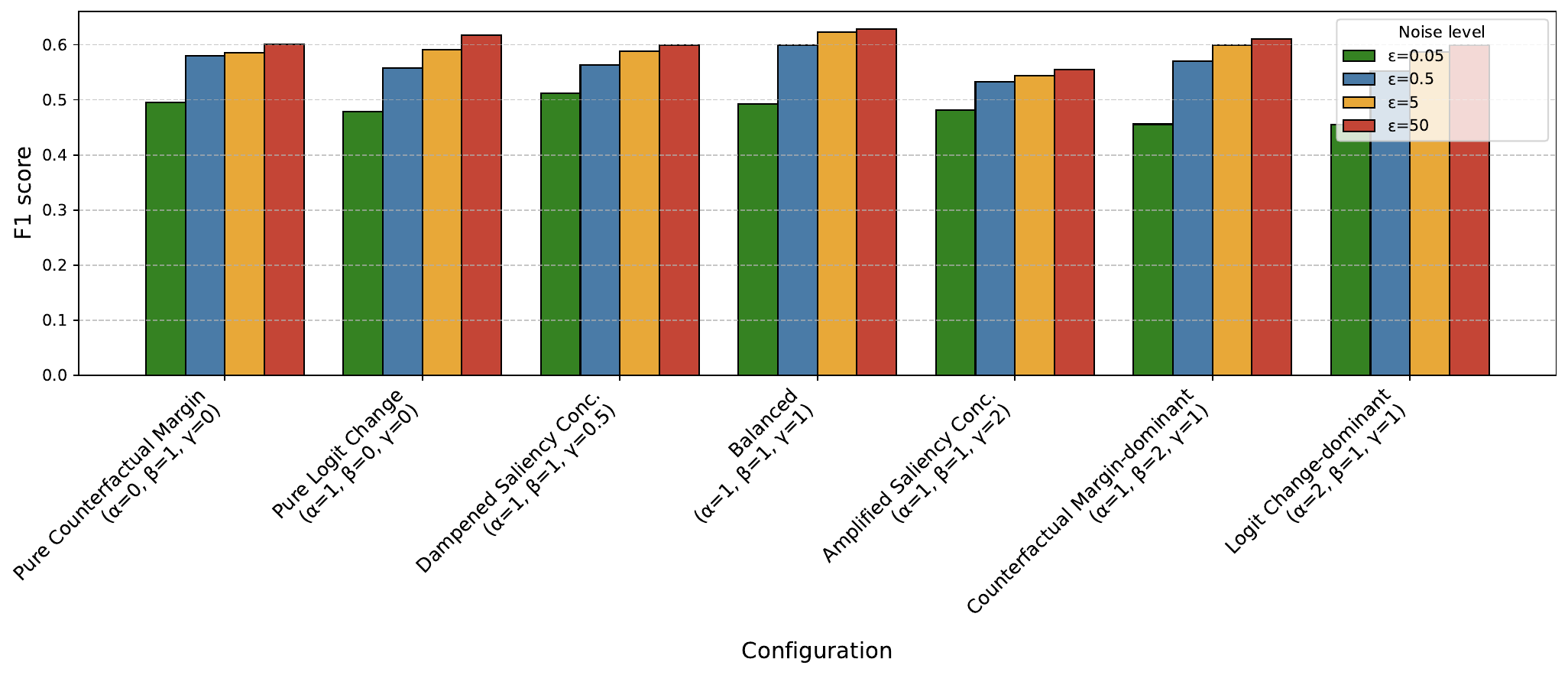}
    \subcaption{Chest X-rays}
    \label{fig:xr_f1_conf_ablation}
    \end{subfigure}
    \caption{Ablation analysis of signal component weights on predictive performance (F1 score) across noise levels ($\varepsilon$). Each configuration (Table \ref{tab:signal_configs}) varies the weights of the logit change ($\alpha$), counterfactual margin ($\beta$), and saliency concentration ($\gamma$) components.
    \label{fig:f1_confs_ablation_analysis}}
    \Description[]{}
\end{figure*}

\subsection{Explainability}\label{sec:appendix_ablation_study_expl}

As shown in Figure \ref{fig:expl_confs_ablation_analysis}, explainability across configurations reveals more pronounced differences than predictive performance, highlighting the importance of component selection for explainability quality.
For the Blood Cells dataset (Figure~\ref{fig:bc_expl_conf_ablation}), the \textit{Balanced} and \textit{Logit Change-dominant} configurations ($\alpha=2, \beta=\gamma=1$) achieved the highest ROAD scores, particularly at relaxed privacy levels ($\varepsilon \geq 5$), reaching a change in the model's confidence above $250\%$.
In contrast, the isolated component configurations (\textit{Pure Logit Change} with $\alpha=1, \beta=0, \gamma=0$, and \textit{Pure Counterfactual Margin} with $\alpha=0, \beta=1, \gamma=0$) yielded substantially lower ROAD scores across all noise levels, indicating that combining both components is beneficial for explainability.

For the Chest X-ray dataset (Figure~\ref{fig:xr_expl_conf_ablation}), the explainability results exhibited a different pattern compared to the other datasets.
The \textit{Dampened Saliency Concentration} configuration ($\alpha=\beta=1,\gamma=0.5$) achieved the highest ROAD scores across all noise levels, reaching approximately $29\%$ at $\varepsilon=50$ and $20\%$ at $\varepsilon=5$.
In contrast, the \textit{Balanced} configuration yielded lower ROAD scores (ranging from $6\%$ to $14\%$), underperforming several other configurations including the isolated component variants.
This suggests that for the Chest X-ray dataset, dampening the saliency concentration scaling allows the model to develop more diffuse but causally relevant attention patterns, which may be better suited to the grayscale nature and larger relevant regions characteristic of chest scans.
The \textit{Amplified Saliency Concentration} configuration ($\alpha=\beta=1,\gamma=2$) performed poorly, achieving the lowest ROAD scores across most noise levels, further indicating that aggressive concentration penalties are detrimental for this dataset.

The explainability results reveal that the optimal configuration may be dataset-dependent.
For the Blood Cells dataset, the \textit{Balanced} configuration provided the most consistent explainability improvements, as the three components offer complementary signals: logit change preserves learning from informative salient regions, counterfactual margin sharpens decision boundaries, and saliency concentration acts as a regularizer, favoring compact and interpretable attention patterns.
For the Chest X-ray dataset, however, dampened concentration scaling ($\gamma=0.5$) yielded superior explainability, likely because pneumonia-related features are more spatially distributed than the localized features in dermoscopic images.
Across all datasets, lower noise levels yielded more confident and focused models with stronger explainability fidelity.

\begin{figure*}[h]
    \centering
    \begin{subfigure}{1\linewidth}
        \centering
        \includegraphics[width=0.80\textwidth]{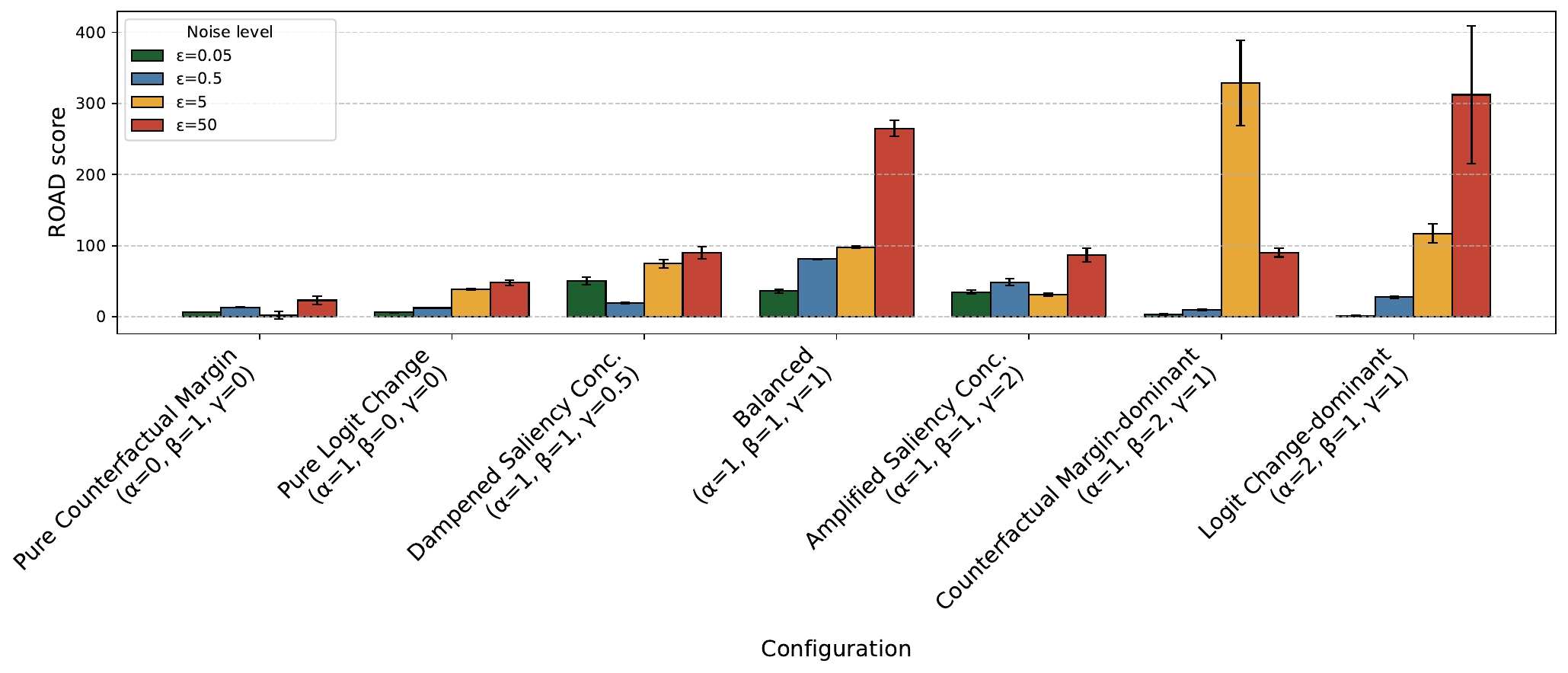}
        \subcaption{Blood cells}
        \label{fig:bc_expl_conf_ablation}
    \end{subfigure}
    \begin{subfigure}{1\linewidth}
        \centering
        \includegraphics[width=0.80\textwidth]{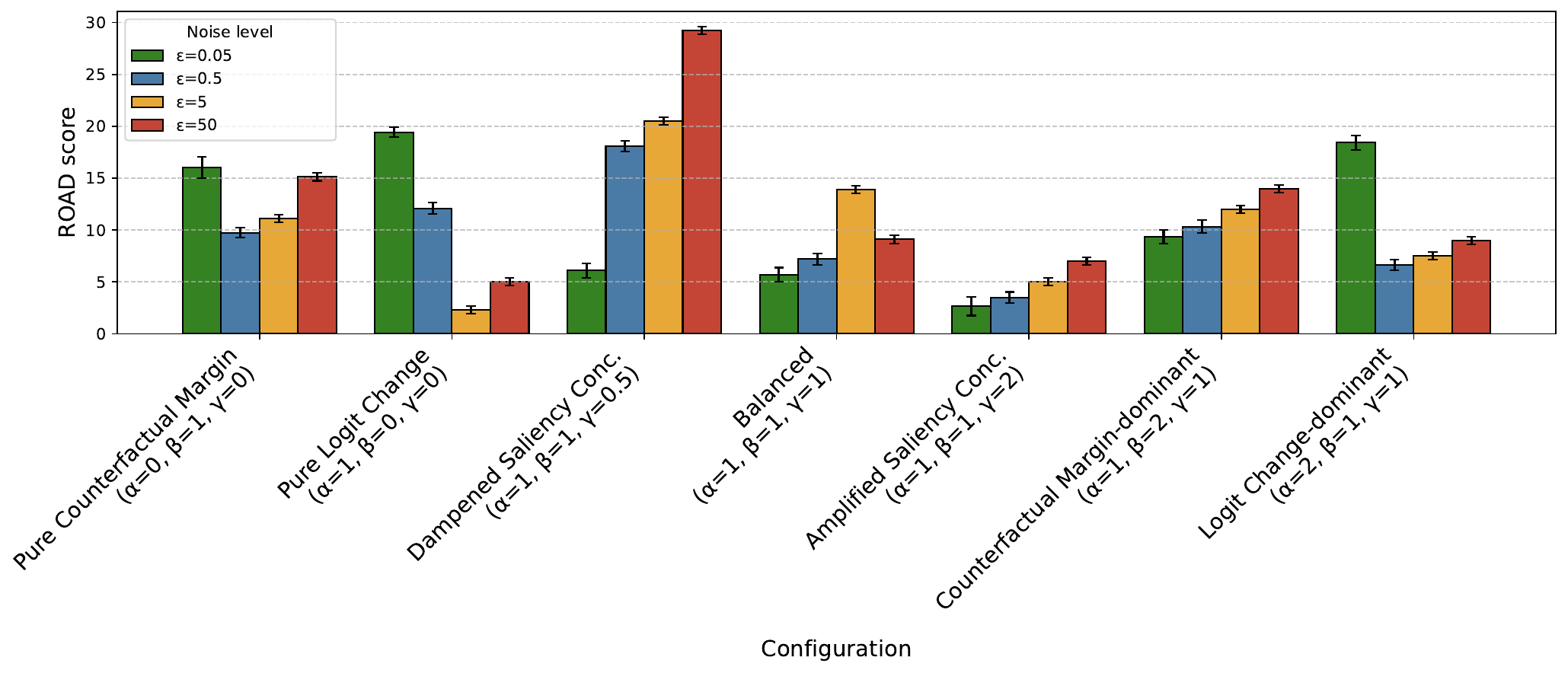}
        \subcaption{Chest X-rays}
        \label{fig:xr_expl_conf_ablation}
    \end{subfigure}
    \caption{Ablation analysis of signal component weights on explainability (ROAD score) across noise levels ($\varepsilon$). Each configuration (Table \ref{tab:signal_configs}) varies the weights of the logit change ($\alpha$), counterfactual margin ($\beta$), and saliency concentration ($\gamma$) components. Higher ROAD scores indicate stronger explainability fidelity.}
    \label{fig:expl_confs_ablation_analysis}
   \Description[]{}
\end{figure*}

\section{Detailed Complexity Analysis}\label{sec:appendix_time_complexity_analysis}
\begin{figure*}
    \centering
    \begin{subfigure}{0.31\linewidth}
        \centering
        \includegraphics[width=0.85\linewidth]{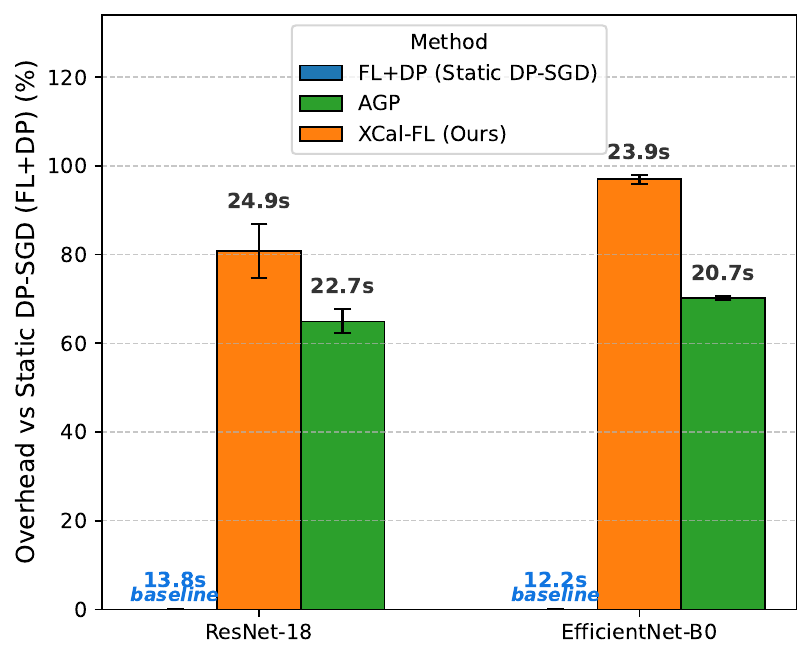}
        \subcaption{Blood cells}
        \label{fig:bc_time_complexity_analysis}
    \end{subfigure}%
    \begin{subfigure}{0.310\linewidth}
    \centering
    \includegraphics[width=0.85\linewidth]{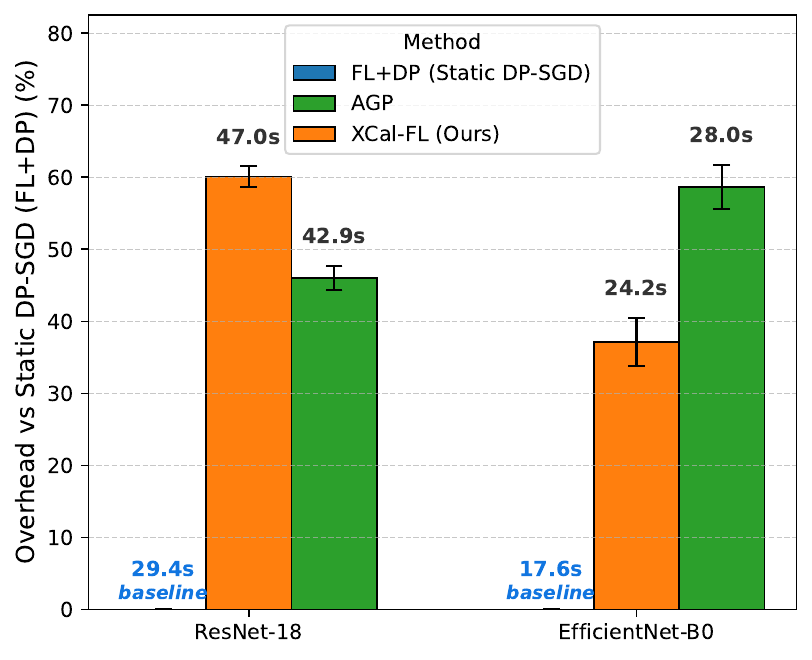}
    \subcaption{Chest X-rays}
    \label{fig:xr_time_complexity_analysis}
    \end{subfigure}%
    \begin{subfigure}{0.310\linewidth}
    \centering
    \includegraphics[width=0.85\linewidth]{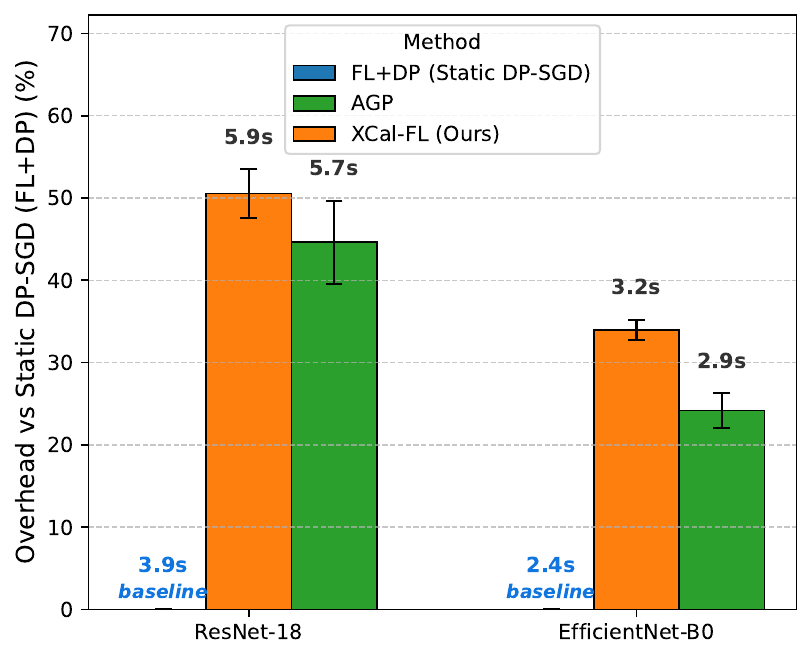}
    \subcaption{Melanoma}
    \label{fig:ml_time_complexity_analysis}
    \end{subfigure}
    \caption{Per-client training overhead (\%) relative to standard static DP-SGD, measured across three medical imaging datasets and two model architectures: ResNet-18 ($P{=}11.7$M) and EfficientNet-B0 ($P{=}5.3$M). Experiments were conducted on a NVIDIA's RTX 2080 Ti GPU machine. Wall-clock per-client times (seconds), averaged over three runs, are annotated above each bar; error bars denote the standard deviation across rounds.}
    \label{fig:time_complexity_analysis}
    \Description[]{}
\end{figure*}

We performed an execution runtime complexity analysis with respect to our used hardware (a machine with a 64GB of RAM and a NVIDIA's RTX 2080 Ti GPU) across all three datasets and our two evaluated model architectures. Specifically, we measured both the wall-clock duration of a single client training round and the incurred overhead relative to the standard client-level DP-SGD training (with static DP noise).

A standard DP-SGD training step costs approximately $3F + \mathcal{O}(P)$ per mini-batch, where $F$ is the cost of a single training-mode forward pass through the model and $P$ is the total number of model parameters: one forward pass contributes $F$, one backward pass contributes $\approx 2F$, and batch-level gradient clipping with noise injection adds $\mathcal{O}(P)$ \citep{abadi2016deeplearningwithdp,epoch2021backwardforwardFLOPratio}.
XCal-FL adds an explainability-driven signal computation per mini-batch, consisting of two inference-mode forward passes (one on the original input and one on the masked input), followed by saliency mask operations on the last convolutional feature map.
An inference-mode forward pass, $F_{\text{inf}}$, costs approximately $0.5$--$0.7 \times F$ since no computation graph is constructed and no intermediate activations are stored.
The per-step cost of XCal-FL is therefore $3F + 2 \cdot F_{\text{inf}} + \mathcal{O}(P)$. Notably, no additional backward pass is required for the explainability signal.
This yields a theoretical overhead of $2 \cdot F_{\text{inf}} / (3F + \mathcal{O}(P))$.

Figure~\ref{fig:time_complexity_analysis} reports the measured overhead across all configurations. 
XCal-FL incurs up to 80\% overhead over static DP-SGD depending on dataset and architecture, which falls within the range of the theoretical overhead bound and is consistent with previous evaluation studies with FL \citep{baumgart2024not}.
AGP~\cite{li2023balancing}, which computes channel-level importance masks via GradCAM (requiring one additional forward and backward pass per mini-batch), exhibits comparable overhead in most configurations.
The absolute per-round difference between XCal-FL and AGP is small, typically 1--3 seconds per client, while XCal-FL yields measurably higher model utility and interpretability fidelity (Sections~\ref{subsec:results_priv_utility_tradeoff} and~\ref{subsec:results_explainability}). 
Practically, this overhead is purely local and does not increase communication cost, which remains the dominant bottleneck in cross-silo federated settings where model transmission over communication links typically requires seconds to minutes per round.
Given the simultaneous gains in accuracy and explainability, this moderate computational overhead represents a favorable trade-off for privacy-sensitive applications, such as clinical imaging.

\end{document}